\documentclass[letterpaper]{article} % DO NOT CHANGE THIS
\usepackage[preprint]{aaai2027}  % DO NOT CHANGE THIS
\usepackage[hyphens]{url}  % DO NOT CHANGE THIS
\usepackage{graphicx} % DO NOT CHANGE THIS
\usepackage{natbib}  % DO NOT CHANGE THIS AND DO NOT ADD ANY OPTIONS TO IT
\usepackage{caption} % DO NOT CHANGE THIS AND DO NOT ADD ANY OPTIONS TO IT
\usepackage{algorithm}
\usepackage{algorithmic}
\usepackage{booktabs}

\usepackage{amsmath,amssymb}
\usepackage{amsthm}
\usepackage{multirow}
\usepackage{placeins}
\usepackage[table]{xcolor}
\usepackage{soul}
\setstcolor{red}   
\usepackage{tikz}

\newtheorem{theorem}{Theorem}

\newtheorem{remark}{Remark}

\newtheorem{corollary}{Corollary}

\newcommand{\filledcircled}[2][\scriptsize]{%
  \tikz[baseline=(char.base)]{
    \node[shape=circle,fill=gray!100,text=white,inner sep=0.8pt] (char) {#1\bfseries #2};
  }%
}

\title{GAUGE: Granularity-Adaptive Counterfactual Gating of Evidence for \\ Incomplete Multimodal Classification}

\author{
     Yunping Shi, En Yu, Kairui Guo, Jie Lu
}
\affiliations{
    Australian Artificial Intelligence Institute (AAII),
    University of Technology Sydney (UTS), Australia \\
    yunping.shi@student.uts.edu.au; \{en.yu-1; kairui.guo; jie.lu\}@uts.edu.au
}

\begin{document}

\maketitle

\begin{abstract}
Multimodal classification typically assumes all modalities are available, yet real-world inputs are often incomplete. Imputation and dynamic fusion can mitigate such incompleteness, but existing methods operate at a coarse modality level and thus cannot retain reliable components while suppressing misleading ones within the same recovered modality, compromising prediction reliability. To address this issue, we propose \textbf{GAUGE}, a lightweight counterfactual gating framework for incomplete multimodal classification. GAUGE first imputes missing modalities with a frozen imputer and encodes observed and recovered inputs uniformly as fine-grained evidence units. Rather than intervening on each unit explicitly, GAUGE scores the counterfactual effect of replacing every unit with a reference representation through prediction-aware Taylor evidence scores, all obtained in a single forward--backward pass. These scores are mapped to continuous gates, which are converted into additive attention-logit biases for unit-wise evidence modulation without altering the backbone architecture. Experiments across six benchmarks demonstrate that GAUGE outperforms strong baselines across diverse incomplete-input settings. Furthermore, a Taylor remainder theoretical analysis characterizes the error of the first-order approximation relative to the exact counterfactual effect, establishing GAUGE as a principled and scalable framework for fine-grained evidence control under modality incompleteness.
\end{abstract}

\section{Introduction}
\label{sec:intro}
Multimodal models have achieved remarkable success in complex perception and reasoning tasks by leveraging complementary information from diverse sources, such as images, text, and sensor data~\cite{liang2022foundations,xu2023multimodal,zhang2025multimodal}. However, they are typically developed under the idealized assumption that all required modalities are available for each input~\cite{ma2021smil,du2024tip,yang2025adapting}. 
In practice, modalities may be missing or inaccessible because of sensor failures, heterogeneous acquisition protocols, privacy constraints, or transmission errors~\cite{sun2024redcore,yao2024drfuse}. Reliable prediction under such incomplete inputs is therefore critical in high-stakes applications such as medical diagnosis~\cite{yao2024drfuse} and autonomous driving~\cite{park2025resilient}.
 
Existing methods for learning with incomplete modalities can be broadly categorized into two paradigms: \emph{recovery-free} and \emph{recovery-based}.
\emph{Recovery-free} methods learn directly from the available modalities~\cite{ma2021smil,wu2024muse,zhang2025synergistic}, avoiding imputation-induced errors but discarding potentially recoverable cross-modal cues. 
In contrast, \emph{recovery-based} approaches explicitly impute missing modalities~\cite{zhang2022m3care,wang2023incomplete} to restore complementary cues, but the reconstructed content may contain low-fidelity, noisy, or semantically inconsistent evidence. Recent quality-aware dynamic fusion~\cite{zhang2023qmf,cao2024pdf} and modality-robustness methods~\cite{wang2024gmd} mitigate this unreliability through weighting,
gradient-guided decoupling, or availability-conditioned parameter switching.
However, these mechanisms operate only at the modality level, applying a shared control decision to the entire encoder output. This coarse-grained approach overlooks within-modality heterogeneity in the reliability of individual components. The limitation is especially pronounced for recovered modalities, where uneven reconstruction errors can cause informative and misleading signals to coexist. Therefore, effective post-recovery control should operate at a finer granularity within each modality, independently modulating its representations according to the contributions to the current prediction.

However, realizing such fine-grained post-recovery control raises three interrelated questions: 
\emph{\textbf{1) At what granularity should control operate?}} 
Modern multimodal architectures employ heterogeneous encoders that produce representation sets with different structures across modalities. Existing mechanisms disregard this encoder-defined structure by assigning a shared control value to the entire representation of each modality~\cite{zhu2026decoupled}. Effective control should therefore operate at the native granularity of each encoder. \emph{\textbf{2) How should control strength be quantified?}}  
The contribution of an individual representation component depends on the current multimodal context and predicted class~\cite{deng2021unified,fan2026enhancing}. Its control strength should be assessed in a prediction-aware manner. A direct criterion is counterfactual replacement, which replaces a component with a reference representation while keeping the remaining representation fixed and measures the magnitude of the resulting change in the predicted-class response.
\emph{\textbf{3) How can control strength be estimated efficiently?}} Although counterfactual replacement yields a decision-aligned score for each component, exact evaluation requires one intervention forward pass per component---\(N\) additional passes for \(N\) components~\cite{fong2017interpretable}. Scalable control therefore requires estimating all component-wise effects using a fixed number of forward--backward scoring passes independent of \(N\).

To address these questions, we propose \textbf{GAUGE} (\textbf{G}ranularity-\textbf{A}daptive co\textbf{U}nterfactual \textbf{G}ating of \textbf{E}vidence), a lightweight counterfactual-attribution gating framework for incomplete multimodal classification. Given an incomplete input, GAUGE recovers the missing modalities with a frozen imputer and encodes each observed or recovered modality into fine-grained \emph{evidence units}\footnote{\emph{Evidence unit} denotes the finest representation emitted by an encoder. It covers fine-grained representations, e.g., image patches, text tokens, and tabular feature embeddings, as well as modality-level embeddings when an encoder emits a single representation.}. This abstraction supports unit-level modulation for multi-unit modalities and reduces to modality-level gating for single-representation modalities. Furthermore, GAUGE defines a representation-level counterfactual effect for each unit by measuring how the predicted-class response changes when that unit is replaced by a reference representation. To avoid exact per-unit interventions, GAUGE derives a signed first-order Taylor approximation and aggregates its channel-wise terms into a non-canceling magnitude score. All scores are obtained with a single ungated forward--backward pass and mapped by a two-scalar gate into additive attention-logit biases, reducing attention to low-scoring units while leaving the Transformer backbone architecture and training objective unchanged. Our contributions are summarized as follows:
\begin{itemize}
\item We formulate a granularity-adaptive evidence control paradigm for incomplete multimodal classification, representing both observed and reconstructed modalities as encoder-emitted evidence units. This abstraction enables fine-grained evidence modulation at the finest available representation level while naturally retaining compatibility with modality-level settings.

\item We propose \textbf{GAUGE}, a lightweight counterfactual-attribution gating framework that modulates evidence according to prediction-aware evidence scores. 
GAUGE formalizes an interpretable unit-wise counterfactual replacement objective and derives a first-order Taylor scoring formulation that obtains all unit-level evidence scores in a single forward--backward pass, 
bypassing the prohibitive overhead of explicit per-unit interventions.

\item Extensive experiments across diverse incomplete-modality settings demonstrate the effectiveness of GAUGE. We further provide a Taylor remainder analysis for the signed first-order approximation and derive an upper bound relating the Taylor evidence score to the magnitude of the exact counterfactual effect.
\end{itemize}

\section{Related Work}
\noindent\textbf{Incomplete Multimodal Learning.}
Existing methods for incomplete multimodal learning broadly follow two paradigms: \emph{recovery-free} and \emph{recovery-based}. 
\emph{Recovery-free} approaches predict directly from available modalities via modality dropout, missing-aware training, prompt tuning, or contrastive learning~\cite{neverova2016moddrop,ma2021smil,lee2023multimodal,wu2024muse,sun2024redcore,wang2018task}, avoiding fabricated content but discarding potentially recoverable cues. 
In contrast, \emph{recovery-based} approaches explicitly reconstruct missing modalities through cross-modal generation, diffusion models, disentanglement, or context-guided completion~\cite{wang2023incomplete,dai2025unbiased,yu2026generalized,liu2025imdr,zhao2026dualstage,huang2026recap}, alongside recent advances in retrieval-augmented prompts~\cite{lang2025ragpt}, feature disentangling~\cite{yao2024drfuse}, pre-training~\cite{du2024tip}, and inference-time selection~\cite{du2026inference}. 
However, these methods predominantly treat each observed or imputed modality as a monolithic decision unit, lacking prediction-aware control at the encoder-emitted feature level. 
GAUGE departs from this paradigm by dynamically re-evaluating both observed and recovered inputs as fine-grained evidence units, rather than treating a reconstructed modality as a static substitute.

\noindent\textbf{Dynamic Multimodal Fusion.}
Dynamic multimodal fusion studies sample-specific modality utility at inference time, such as input-conditioned routing~\cite{xue2023dynamic}, quality-aware weighting~\cite{zhang2023qmf}, predicted modality contributions~\cite{cao2024pdf}, gradient-based decoupling~\cite{wang2024gmd}, and cross-modal enhancement~\cite{chen2026smcir}. 
However, these methods operate exclusively at the coarse modality level, applying a uniform scalar weight, routing score, or gating decision to each modality. 
This formulation conflicts with Transformer encoders, whose fine-grained internal units exhibit heterogeneous prediction relevance that becomes especially restrictive when handling recovered modalities. Our research advances this by shifting dynamic fusion from modality-level weighting to fine-grained unit-level modulation.

\noindent\textbf{Feature Attribution and Counterfactual Reasoning.}
Gradient-based attribution methods characterize reference-based or local changes in model outputs with respect to input features~\cite{sundararajan2017axiomatic,shrikumar2017deeplift,ancona2018unified,li2025causality}. 
In particular, a first-order Taylor expansion around the current representation provides a local approximation of the prediction change induced by replacing a feature with a reference~\cite{deng2021unified}. 
Unlike the post-hoc explainers, GAUGE repurposes such attribution signals as an in-the-loop mechanism, computing unit-level Taylor evidence scores in a single backward pass against a prediction-aware objective to directly drive the attention bias of the underlying model.

\begin{figure*}[th]
  \centering
    \includegraphics[width=0.99\textwidth]{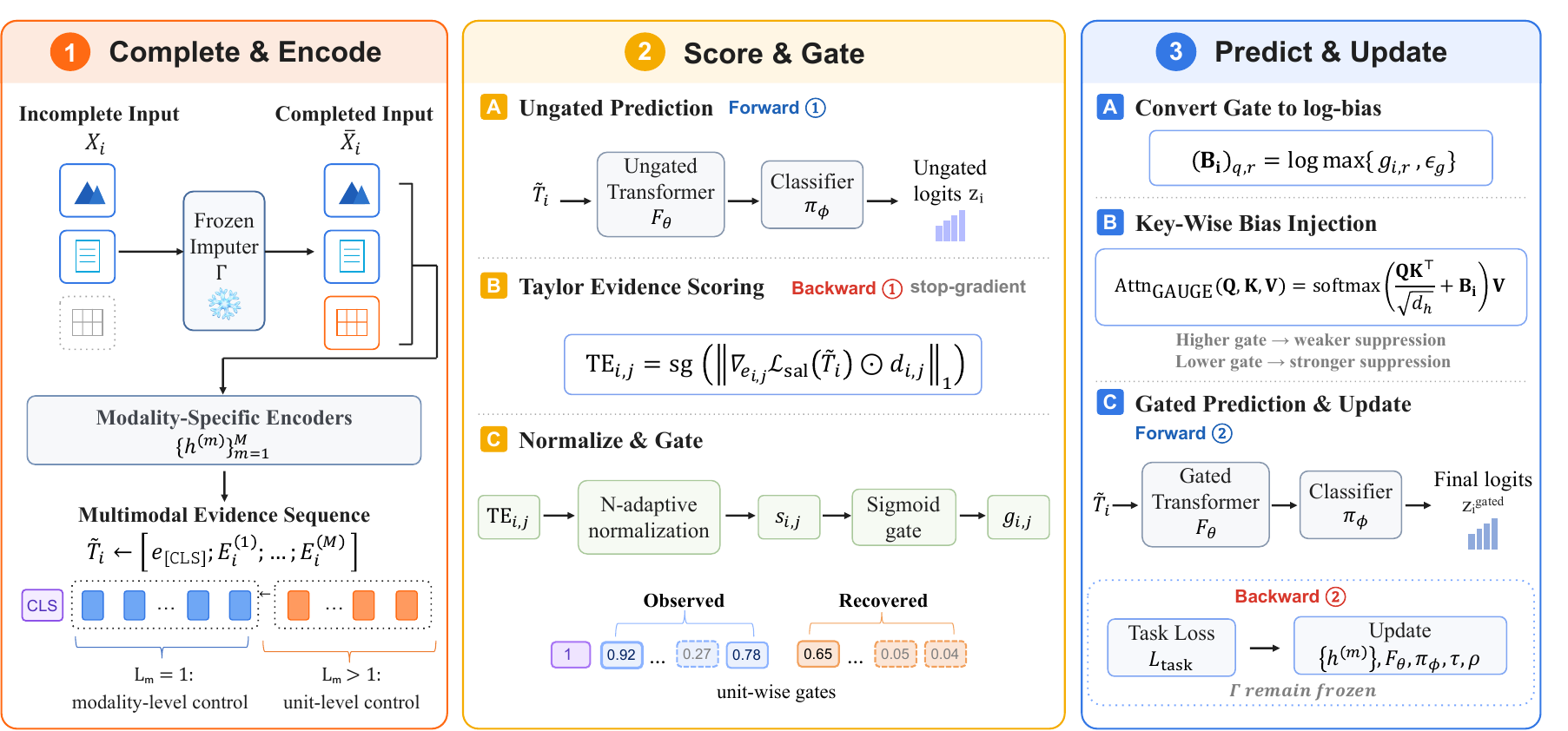}
\caption{Overview of the GAUGE framework:
(1)~complete \& encode;
(2)~score \& gate, where Taylor evidence scores are computed for all
evidence units in a single backward pass; and
(3)~predict \& update via additive attention-bias injection.}
  \label{fig:framework}
\end{figure*}

\section{Methodology}
\label{sec:methodology}

\noindent\textbf{Task Description.}
Let $\mathcal{D}=\{(X_i,y_i)\}_{i=1}^{|\mathcal{D}|}$ be a dataset for incomplete multimodal classification, where $y_i\in\{1,\ldots,K\}$ is the class label among $K$ target classes. 
For each sample $i$ under the incomplete-modality setting, only a subset of modalities $\mathcal{O}_i\subseteq\{1,\ldots,M\}$ is observed, where $M$ denotes the total number of modalities. 
The incomplete input is thus represented as $X_i=\{x_i^{(m)}\}_{m\in\mathcal{O}_i}$, where $x_i^{(m)}$ is the observation from modality $m$, while the missing modalities are indexed by the set $\mathcal{U}_i=\{1,\ldots,M\}\setminus\mathcal{O}_i$. A pre-trained multimodal imputer $\Gamma$ is employed to reconstruct the missing inputs. 
Specifically, each missing modality \(u\in\mathcal{U}_i\) is recovered as
\(\tilde{x}_i^{(u)}=\Gamma_u(X_i)\).
The formulation also extends to feature-level missingness within a partially observed modality, as detailed in Appendix~\ref{app:feature-missing}.
Merging the observed and reconstructed inputs yields the completed input $\bar{X}_i=\{\bar{x}_i^{(m)}\}_{m=1}^{M}$, where $\bar{x}_i^{(m)}=x_i^{(m)}$ if $m\in\mathcal{O}_i$, and $\bar{x}_i^{(m)}=\tilde{x}_i^{(m)}$ otherwise. 
The final goal is to predict $y_i$ given the completed input $\bar{X}_i$.

\subsection{Overview of GAUGE}
\label{subsec:overview}
To achieve this goal, we propose \textbf{GAUGE}, a principled and scalable framework for granularity-adaptive evidence control. 
As illustrated in Fig.~\ref{fig:framework} and summarized in Algorithm~\ref{alg:GAUGE}, GAUGE consists of three phases:

\noindent\textbf{Phase \filledcircled{1} -- Complete \& Encode.}
A frozen imputer $\Gamma$ reconstructs missing modalities to produce the completed multimodal input $\bar{X}_i$. 
Modality-specific encoders $h^{(m)}$ then map each completed modality $\bar{x}_i^{(m)}$ to one or more evidence units, which are concatenated with a learnable $[\mathtt{CLS}]$ token to form the multimodal sequence $\tilde{T}_i$.

\noindent\textbf{Phase \filledcircled{2} -- Score \& Gate.}
Given \(\tilde{T}_i\), GAUGE performs an ungated forward pass to obtain the logits \(z_i\) and the ungated predicted class \(\hat{y}_i\). With \(\hat{y}_i\) fixed, a single backward pass on the saliency objective \(\mathcal{L}_{\mathrm{sal}}(\tilde{T}_i)=-z_{i,\hat{y}_i}\)
yields a magnitude-based \emph{Taylor evidence score} \(\mathrm{TE}_{i,j}\) for each evidence unit along the replacement direction \(d_{i,j}=e'_{i,j}-e_{i,j}\).
GAUGE then applies an \(N\)-adaptive score normalization to obtain \(s_{i,j}\) and maps it to a continuous gate \(g_{i,j}\) using two shared learnable scalars \((\tau,\rho)\).

\noindent\textbf{Phase \filledcircled{3} -- Predict \& Update.}
The resulting gates \(g_{i,j}\) are converted into an additive attention-logit bias matrix \(\mathbf{B}_i\), applied at every Transformer block without altering the underlying backbone architecture.
A gated forward pass on the same sequence \(\tilde{T}_i\) then produces
the final logits \(z_i^{\mathrm{gated}}\) and the final predicted class
\(\hat{y}_i^{\mathrm{final}}\).
During training, we minimize the cross-entropy loss on the gated prediction, which updates the encoders \(\{h^{(m)}\}_{m=1}^{M}\), the backbone \(F_\theta\), the classifier \(\pi_\phi\), and the gate parameters \((\tau,\rho)\); the imputer \(\Gamma\) remains frozen.

\subsection{Granularity-Adaptive Evidence Units}
\label{subsec:evidence_units}
To instantiate the fine-grained evidence control, GAUGE first represents each observed or recovered modality at the finest granularity exposed by its encoder.
For modality \(m\), the encoder \(h^{(m)}\) maps the completed modality
\(\bar{x}_i^{(m)}\) to \(L_m\) evidence units,
\(E_i^{(m)}=h^{(m)}(\bar{x}_i^{(m)})
=[\,e_{i,1}^{(m)},\ldots,e_{i,L_m}^{(m)}\,]
\in\mathbb{R}^{L_m\times C}\), where \(C\) denotes the shared hidden
dimension. When \(L_m>1\), the units may correspond to image patches, text tokens, or tabular feature embeddings; when \(L_m=1\), the modality is represented by a single global embedding.
The evidence units from all modalities are concatenated in a fixed modality
order as
\(T_i=[\,E_i^{(1)};\ldots;E_i^{(M)}\,]
=[\,e_{i,1},\ldots,e_{i,N}\,]\in\mathbb{R}^{N\times C}\),
where \(e_{i,j}\) denotes the \(j\)-th unit in the flattened multimodal
sequence and
\(N=\sum_{m=1}^{M}L_m\).
Because each \(L_m\) is fixed by the encoder configuration, \(N\) is
constant within each dataset.
We prepend a learnable \([\mathtt{CLS}]\) token to form \(\tilde{T}_i=[\,e_{[\mathtt{CLS}]};T_i\,]\in\mathbb{R}^{(N+1)\times C}\); modality and positional embeddings are incorporated but omitted from the notation for brevity. Each encoder-emitted representation thus forms a separate evidence unit, allowing units from the same modality to receive different Taylor evidence scores and gates, while the formulation reduces to modality-level gating whenever \(L_m=1\).

\subsection{Prediction-Aware Counterfactual Formulation}
\label{subsec:counterfactual}
\paragraph{\textit{1) Ungated prediction and saliency objective.}}
To determine each unit's evidence strength in a prediction-aware manner, GAUGE performs an ungated forward pass on \(\tilde{T}_i\) and uses the resulting predicted class as the fixed target for evidence scoring.
The Transformer backbone \(F_\theta\) processes the input sequence \(\tilde{T}_i\), and the classifier head \(\pi_\phi\) maps the resulting \([\mathtt{CLS}]\) representation to class logits, with the corresponding class probabilities defined as:
\begin{equation}
z_i
=
\pi_\phi\!\left(
F_\theta(\tilde{T}_i)_{[\mathtt{CLS}]}
\right),
\qquad
p_i=\mathrm{softmax}(z_i),
\label{eq:logits}
\end{equation}
where \(z_i\in\mathbb{R}^{K}\) denotes the logit vector over the \(K\) classes and \(z_{i,k}\) denotes its \(k\)-th entry.
The ungated predicted class is \(\hat{y}_i=\arg\max_k z_{i,k}\).
We define the prediction-aware saliency objective as
\(\mathcal{L}_{\mathrm{sal}}(\tilde{T}_i)=-z_{i,\hat{y}_i}\), where \(\hat{y}_i\) is fixed during evidence-unit replacement and differentiation.
We adopt this logit-based objective rather than cross-entropy because the latter may yield weak attribution gradients for confident predictions; a detailed comparison is provided in Appendix~\ref{app:logit-vs-ce}.

\begin{algorithm}[t]
\footnotesize
\caption{\textbf{GAUGE}: Minibatch Training}
\label{alg:GAUGE}
\begin{algorithmic}[1]
\REQUIRE Minibatch $\{(X_i,y_i)\}_{i\in\mathcal B}$;
frozen imputer $\Gamma$; trainable parameters
$\Theta=\bigl\{\{h^{(m)}\}_{m=1}^{M},
F_\theta,\pi_\phi,\tau,\rho\bigr\}$.
\ENSURE Updated parameters $\Theta$.
\FOR{$i\in\mathcal B$}
  \STATE Form the completed input $\bar X_i$ from $X_i$
  using frozen $\Gamma$.
  \STATE Compute
  $E_i^{(m)}\leftarrow h^{(m)}(\bar x_i^{(m)})$
  for $m=1,\ldots,M$.
  \STATE $\tilde T_i\leftarrow
  [\,e_{[\mathtt{CLS}]};E_i^{(1)};\ldots;E_i^{(M)}\,]$
  in fixed modality order.
  \STATE Compute $z_i$ (Eq.~\ref{eq:logits}) and set
  $\hat y_i\leftarrow\arg\max_k z_{i,k}$.
  \COMMENT{$\hat y_i$ fixed}
\ENDFOR
\STATE $\mathcal L_{\mathrm{sal}}^{\mathcal B}
\leftarrow-\sum_{i\in\mathcal B}z_{i,\hat y_i}$.
\STATE Differentiate $\mathcal L_{\mathrm{sal}}^{\mathcal B}$ once to obtain
all $\nabla_{e_{i,j}}\mathcal L_{\mathrm{sal}}(\tilde T_i)$
for $i\in\mathcal B$ and $j=1,\ldots,N$;
do not accumulate gradients for $\Theta$.

\FOR{$i\in\mathcal B$}
  \FOR{$j=1,\ldots,N$}
    \STATE $e'_{i,j}\leftarrow\mathbf 0$;\quad
    $d_{i,j}\leftarrow e'_{i,j}-e_{i,j}$.
\STATE $\mathrm{TE}_{i,j}\leftarrow
\operatorname{sg}\!\left(
\left\|
\nabla_{e_{i,j}}\mathcal L_{\mathrm{sal}}(\tilde T_i)
\odot d_{i,j}
\right\|_1
\right)$.
  \ENDFOR
  \STATE Compute $s_{i,j}$ and $g_{i,j}$ for $j=1,\ldots,N$
  by Eqs.~\ref{eq:zscore} and~\ref{eq:gate}.
  \STATE Set $g_{i,0}\leftarrow1$ and construct
  $\mathbf B_i$ (Eq.~\ref{eq:bias_def}).
  \STATE $z_i^{\mathrm{gated}}\leftarrow
  \pi_\phi\!\left(
  F_\theta(\tilde T_i;\mathbf B_i)_{[\mathtt{CLS}]}
  \right)$.
  \STATE $p_i^{\mathrm{gated}}\leftarrow
  \operatorname{softmax}(z_i^{\mathrm{gated}})$.
\ENDFOR
\STATE Compute $\mathcal L_{\mathrm{task}}^{\mathcal B}$
(Eq.~\ref{eq:loss}), backpropagate, and update $\Theta$.
\end{algorithmic}
\end{algorithm}

\noindent
\paragraph{\textit{2) Representation-level counterfactual replacement.}}
For the \(j\)-th evidence unit \(e_{i,j}\) in the flattened multimodal sequence \(T_i\), where \(j\in\{1,\ldots,N\}\), we construct a representation-level counterfactual by replacing \(e_{i,j}\) with a reference representation \(e'_{i,j}\). Specifically, we set \(e'_{i,j}=\mathbf{0}\in\mathbb{R}^{C}\). The resulting evidence sequence is denoted by
\(T_i^{(j\leftarrow e'_{i,j})} = [\,e_{i,1},\ldots,e'_{i,j},\ldots,e_{i,N}\,]\),
and its \([\mathtt{CLS}]\)-prepended counterpart by
\(\tilde{T}_i^{(j\leftarrow e'_{i,j})}
= [\,e_{[\mathtt{CLS}]}; 
T_i^{(j\leftarrow e'_{i,j})}\,]\).
The \([\mathtt{CLS}]\) token is never replaced. The exact counterfactual effect of unit \(j\) is then defined as
\begin{equation}
\mathrm{CEff}_{i,j}^{\mathrm{exact}}=
\mathcal{L}_{\mathrm{sal}}\!\left(\tilde{T}_i^{(j\leftarrow e'_{i,j})}\right)
-
\mathcal{L}_{\mathrm{sal}}(\tilde{T}_i).
\label{eq:exact_ceff}
\end{equation}
However, evaluating Eq.~\eqref{eq:exact_ceff} for all \(N\) evidence units requires \(N\) additional forward passes per sample, motivating the pass-efficient Taylor scoring formulation introduced next.

\subsection{Efficient Taylor Gating via Attention Bias}
\label{subsec:attn_bias_inject}
\paragraph{\textit{1) First-order counterfactual approximation.}}
To overcome this bottleneck, GAUGE avoids per-unit intervention forward passes by approximating each exact counterfactual effect with a signed first-order Taylor term.
For evidence unit \(e_{i,j}\), let $d_{i,j}=e'_{i,j}-e_{i,j}$ denote the replacement direction from $e_{i,j}$ to its reference representation $e'_{i,j}$.
A first-order Taylor expansion of \(\mathcal{L}_{\mathrm{sal}}\) around
\(e_{i,j}\) along \(d_{i,j}\) yields
\begin{equation}
\mathrm{CEff}_{i,j}^{\mathrm{exact}}
=
\left\langle
\nabla_{e_{i,j}}\mathcal{L}_{\mathrm{sal}}(\tilde{T}_i),
d_{i,j}
\right\rangle
+
R_{i,j},
\label{eq:taylor_expansion}
\end{equation}
where \(R_{i,j}\) denotes the Taylor remainder.
We denote the signed first-order term by
\(a_{i,j}=
\langle
\nabla_{e_{i,j}}\mathcal{L}_{\mathrm{sal}}(\tilde{T}_i),
d_{i,j}
\rangle\).
Thus, Eq.~\eqref{eq:taylor_expansion} can be written as
\(\mathrm{CEff}_{i,j}^{\mathrm{exact}}=a_{i,j}+R_{i,j}\).
Assuming that \(\mathcal{L}_{\mathrm{sal}}\) is twice continuously
differentiable along the replacement segment and that the spectral norm
of its Hessian with respect to \(e_{i,j}\) is bounded by \(H_{i,j}\)
throughout this segment, the Taylor remainder satisfies \(|R_{i,j}|\le\tfrac{H_{i,j}}{2}\|d_{i,j}\|_2^2\). The full derivation is provided in Appendix~\ref{app:te}.

\noindent
\paragraph{\textit{2) Taylor evidence scoring.}}
The signed first-order term \(a_{i,j}\) is the sum of the channel-wise
gradient--direction products, so contributions with opposite signs can
cancel.
To obtain a non-canceling magnitude score, we instead
aggregate the absolute channel-wise products using the \(\ell_1\) norm,
inspired by gradient-based attribution methods~\cite{
sundararajan2017axiomatic,
shrikumar2017deeplift,
deng2021unified}:
\begin{equation}
\begin{aligned}
\mathrm{TE}_{i,j}
&=
\operatorname{sg}\!\left(
\sum_{c=1}^{C}
\left|
\frac{\partial \mathcal{L}_{\mathrm{sal}}(\tilde{T}_i)}
{\partial e_{i,j,c}}\,
d_{i,j,c}
\right|
\right) \\
&=
\operatorname{sg}\!\left(
\left\|
\nabla_{e_{i,j}}\mathcal{L}_{\mathrm{sal}}(\tilde{T}_i)
\odot d_{i,j}
\right\|_1
\right).
\end{aligned}
\label{eq:taylor_score}
\end{equation}

Here, \(\operatorname{sg}(\cdot)\) denotes the stop-gradient operator: it acts as the identity in the forward pass and has zero derivative during backpropagation. It therefore leaves both \(\mathrm{TE}_{i,j}\) and the bound below numerically unchanged, while preventing task-loss gradients from propagating through the gradient computation used to form the score during training.
Applying the triangle inequality to their sum yields \(|a_{i,j}|\le\mathrm{TE}_{i,j}\).
Combined with the second-order remainder bound, this leads to the following upper bound on the magnitude of the exact counterfactual effect:
\begin{equation}
|\mathrm{CEff}_{i,j}^{\mathrm{exact}}|
\le
\mathrm{TE}_{i,j}
+
\frac{H_{i,j}}{2}
\|d_{i,j}\|_2^2.
\label{eq:te_ceff_bound}
\end{equation}
This bound is established in Corollary~\ref{cor:upper} of Appendix~\ref{app:te-l1}. Accordingly, \(\mathrm{TE}_{i,j}\) quantifies the aggregate magnitude of
the channel-wise first-order contributions induced by replacing unit
\(j\) along \(d_{i,j}\).
All \(N\) gradients \(\{\nabla_{e_{i,j}}\mathcal{L}_{\mathrm{sal}}(\tilde{T}_i)\}_{j=1}^{N}\) are obtained simultaneously from a single ungated backward pass, enabling
GAUGE to compute all evidence-unit scores without the \(N\) additional intervention forward passes required for exact counterfactual evaluation.

\paragraph{\textit{3) \(N\)-adaptive score normalization.}}
Raw Taylor evidence scores can differ substantially in scale across samples because their magnitudes depend jointly on the prediction gradient and the
replacement direction, both of which vary across samples.
Since the gate uses a threshold and temperature shared across units and samples, this variation can hinder consistent gate calibration
across samples.
GAUGE therefore normalizes the \(N\) Taylor evidence scores of each sample using the following \(N\)-adaptive rule before gate construction:
\begin{equation}
s_{i,j}
=
\begin{cases}
\mathrm{TE}_{i,j},
& N\le 3,\\[3pt]
\dfrac{\mathrm{TE}_{i,j}-\mu_i}
{\max(\sigma_i,\epsilon)},
& N>3,
\end{cases}
\label{eq:zscore}
\end{equation}
where for \(N>3\),
$
\mu_i
=
\frac{1}{N}
\sum_{\ell=1}^{N}\mathrm{TE}_{i,\ell}$, $
\sigma_i^2
=
\frac{1}{N-1}
\sum_{\ell=1}^{N}
(\mathrm{TE}_{i,\ell}-\mu_i)^2,
$
and \(\epsilon>0\) is a small constant for numerical stability.
For \(N>3\), this per-sample standardization removes the sample-specific offset and scale, making \(s_{i,j}\) a relative score within the sample.
For \(N\le3\), statistics based on at most three scores can yield highly constrained standardized values and remove informative absolute magnitude, so GAUGE retains the raw scores.
Because \(N\) is fixed by the encoder layout, each model uses the same normalization branch throughout training and inference. Appendix~\ref{app:norm-ablation} evaluates this normalization rule.

\paragraph{\textit{4) Continuous gate construction.}}
Given the evidence score \(s_{i,j}\), GAUGE maps it to a continuous gate
\(g_{i,j}\in(0,1)\):
\begin{equation}
g_{i,j}
=
\operatorname{sigmoid}\!\left(
\frac{s_{i,j}-\tau}{\exp(\rho)}
\right).
\label{eq:gate}
\end{equation}
Here, \(\tau\in\mathbb{R}\) is a learnable threshold, with \(g_{i,j}=0.5\) at \(s_{i,j}=\tau\), and
\(\rho\in\mathbb{R}\) parameterizes the positive temperature \(\exp(\rho)\), which controls the transition sharpness.
This monotonic mapping assigns larger gates to higher-scoring units, enabling graded modulation rather than hard selection.
The scalars \((\tau,\rho)\) are shared across all samples and evidence units, adding only two parameters regardless of the modality or unit count.
The \([\mathtt{CLS}]\) token, indexed by \(j=0\), serves as the aggregation anchor rather than an evidence unit and is therefore kept ungated (\(g_{i,0}=1\)).

\paragraph{\textit{5) Key-wise attention-bias injection.}}
GAUGE converts each gate into an additive attention-logit bias at the key position in each Transformer block, augmenting standard scaled dot-product attention~\cite{vaswani2017attention}:
\begin{equation}
\mathrm{Attn}_{\mathrm{GAUGE}}
(\mathbf{Q},\mathbf{K},\mathbf{V})
=
\mathrm{softmax}\!\left(
\frac{\mathbf{Q}\mathbf{K}^{\top}}{\sqrt{d_h}}
+\mathbf{B}_i
\right)\mathbf{V},
\label{eq:attn_bias}
\end{equation}
where \(d_h\) denotes the key dimension of each attention head, and
\(\mathbf{B}_i\in\mathbb{R}^{(N+1)\times(N+1)}\) is defined as
\begin{equation}
(\mathbf{B}_i)_{q,r}=\log\max\{g_{i,r},\epsilon_g\},
\qquad q,r\in\{0,\ldots,N\},
\label{eq:bias_def}
\end{equation}
where \(q\) and \(r\) index the query and key positions, respectively, and \(\epsilon_g\in(0,1)\) is a small constant that prevents \(\log 0\) and keeps the bias finite. 
Index \(0\) corresponds to the \([\mathtt{CLS}]\) token, whose gate is fixed to \(g_{i,0}=1\), yielding zero attention bias.
Because \((\mathbf{B}_i)_{q,r}\) depends only on the key index \(r\), the same unit-level bias is applied at every query position and shared across attention heads. Since \(0<\max\{g_{i,r},\epsilon_g\}\le1\), all entries of
\(\mathbf{B}_i\) are non-positive. Adding \(\log\max\{g_{i,r},\epsilon_g\}\) multiplies the unnormalized attention
weight of key \(r\) by \(\max\{g_{i,r},\epsilon_g\}\).
Smaller gates therefore more strongly down-weight the corresponding keys, while preserving the sequence length and backbone architecture.

\subsection{Training and Inference}
\label{subsec:training}
GAUGE follows the same imputation-and-gating pipeline during training and inference, with the imputer \(\Gamma\) kept frozen. 
During training, given a minibatch \(\mathcal B\), we minimize the cross-entropy loss on the gated prediction
\(p_i^{\mathrm{gated}}=\operatorname{softmax}(z_i^{\mathrm{gated}})\):
\begin{equation}
\mathcal{L}_{\mathrm{task}}^{\mathcal B}
=
-\frac{1}{|\mathcal B|}
\sum_{i\in\mathcal B}
\log p_i^{\mathrm{gated}}(y_i).
\label{eq:loss}
\end{equation}
This objective updates the modality encoders \(\{h^{(m)}\}_{m=1}^{M}\), Transformer backbone \(F_\theta\),
classifier \(\pi_\phi\), and gate parameters \((\tau,\rho)\).
The Taylor evidence scores are detached before normalization and gate construction, so task optimization does not backpropagate through the saliency-gradient computation or require second-order differentiation.
Each training step comprises one ungated forward--backward scoring pass and one gated forward--backward optimization pass.
At inference, the same ungated forward--backward scoring pass is followed by one gated forward pass to produce the prediction.

\begin{table*}[!t]
\centering
\setlength{\tabcolsep}{3pt}
\renewcommand{\arraystretch}{1.1}
\resizebox{\textwidth}{!}{%
\begin{tabular}{l ccc ccc ccc ccc ccc ccc}
\toprule
\multirow{3}{*}{Model}
& \multicolumn{3}{c}{PolyMNIST Acc (\%) $\uparrow$}
& \multicolumn{3}{c}{MST Acc (\%) $\uparrow$}
& \multicolumn{3}{c}{CelebA Acc (\%) $\uparrow$}
& \multicolumn{3}{c}{DVM Acc (\%) $\uparrow$}
& \multicolumn{3}{c}{CAD AUC (\%) $\uparrow$}
& \multicolumn{3}{c}{Infarction AUC (\%) $\uparrow$} \\
\cmidrule(lr){2-4} \cmidrule(lr){5-7} \cmidrule(lr){8-10}
\cmidrule(lr){11-13} \cmidrule(lr){14-16} \cmidrule(lr){17-19}
& \multicolumn{3}{c}{Missing Rate $\eta$}
& \multicolumn{6}{c}{Missing Modalities}
& \multicolumn{9}{c}{Missing Tabular Rate $\gamma$} \\
\cmidrule(lr){2-4} \cmidrule(lr){5-10} \cmidrule(lr){11-19}
& 0 & 0.6 & 0.8
& $\emptyset$ & \{S,T\} & \{M,T\}
& $\emptyset$ & \{I\} & \{T\}
& 0 & 0.9 & 1
& 0 & 0.9 & 1
& 0 & 0.9 & 1 \\
\midrule
\multicolumn{19}{c}{\textit{(a) Recovery-based Methods for Missing Modality}} \\
\midrule
MultiAE          & 99.77 & 95.36 & 84.39 & 99.96 & 97.00 & 81.60 & 89.98 & 15.51 & 89.71 & - & - & - & - & - & - & - & - & - \\
MultiAE$^\dagger$ & 99.94 & 97.50 & 89.86 & 99.87 & \underline{98.33} & 83.44 & 88.66 & 72.04 & 87.44 & - & - & - & - & - & - & - & - & - \\
MoPoE            & 99.79 & 93.94 & 79.84 & 99.62 & 90.86 & 79.01 & 38.97 & 13.91 & 34.84 & - & - & - & - & - & - & - & - & - \\
MoPoE$^\dagger$  & 99.63 & 96.81 & 87.06 & 99.39 & 96.50 & 82.54 & 68.22 & 56.90 & 65.75 & - & - & - & - & - & - & - & - & - \\
M3Care           & 99.93 & 56.66 & 40.53 & \underline{99.99} & 16.03 & 9.34 & 92.33 & 99.92 & 51.75 & 98.44 & - & 11.92 & \textbf{85.62} & - & 64.99 & 70.61 & - & 70.53 \\
M3Care$^\dagger$ & \underline{99.99} & 97.27 & 87.92 & 99.98 & 98.27 & \underline{85.16} & 98.73 & 97.14 & \underline{91.32} & 98.94 & - & \underline{93.43} & 72.48 & - & \textbf{72.48} & 83.27 & - & 68.44 \\
OnlineMAE        & \textbf{100.00} & 98.29 & 90.09 & 99.90 & 98.14 & 84.14 & 86.67 & 86.67 & 86.67 & 90.92 & - & 89.90 & 85.22 & - & \underline{70.96} & 84.05 & - & 61.39 \\
\midrule
\multicolumn{19}{c}{\textit{(b) Recovery-free Methods for Missing Modality}} \\
\midrule
ModDrop          & 99.97 & 97.66 & 88.44 & \textbf{100.00} & 98.21 & 82.47 & 99.93 & 99.93 & 87.32 & 99.02 & 89.80 & 87.97 & 85.10 & 70.77 & 69.18 & 84.76 & 72.06 & \underline{72.16} \\
MTL              & 99.97 & 98.43 & 91.14 & 99.96 & \textbf{98.60} & 84.37 & 99.69 & 99.26 & 89.38 & \textbf{99.44} & 93.38 & 92.32 & 84.87 & \underline{73.08} & 70.23 & 83.59 & 69.82 & 69.90 \\
MAP              & 99.86 & 43.00 & 23.19 & \textbf{100.00} & 9.83 & 10.13 & \underline{99.98} & 99.93 & 85.33 & 98.86 & 74.88 & 63.15 & 84.39 & 71.39 & 70.11 & 84.62 & 68.47 & 69.17 \\
MAP$^\dagger$    & \underline{99.99} & 96.74 & 76.20 & \underline{99.99} & 97.84 & 11.36 & \textbf{100.00} & \underline{99.99} & 86.06 & \underline{99.37} & 92.43 & 91.15 & \underline{85.26} & 70.36 & 68.76 & 85.49 & 71.38 & 70.81 \\
MUSE             & 99.93 & 94.73 & 77.56 & 99.86 & 97.14 & 35.96 & 99.93 & 99.86 & 88.35 & 96.86 & - & 1.64 & 83.47 & - & 53.23 & 84.40 & - & 66.78 \\
\midrule
\multicolumn{19}{c}{\textit{(c) Dynamic Recovery Method for Missing Modality}} \\
\midrule
DyMo$_c$ & \textbf{100.00} & \textbf{99.79} & \underline{96.74} & \underline{99.99} & 97.61 & 84.11 & \textbf{100.00} & \textbf{100.00} & 87.26 & 99.34 & \underline{93.48} & 92.87 & 84.12 & 71.29 & 68.31 & \underline{86.96} & \underline{73.70} & 71.36 \\
\rowcolor{gray!15}
GAUGE & \textbf{100.00} & \underline{99.67} & \textbf{97.61} & \underline{99.99} & 97.86 & \textbf{85.26} & \textbf{100.00} & \textbf{100.00} & \textbf{93.16} & 99.30 & \textbf{94.79} & \textbf{94.42} & 84.68 & \textbf{74.20} & 70.41 & \textbf{87.12} & \textbf{74.37} & \textbf{72.86} \\
\bottomrule
\end{tabular}%
}
\caption{Comparison of incomplete-modality methods across six benchmarks.
\textbf{Bold} indicates the best result, and \underline{underline} the
second best; ties are both marked.
Except for the reproduced DyMo$_c$, baseline results are from
\citet{du2026inference} under the same missingness protocol; $\dagger$ denotes their incomplete-simulation
variants, and ``-'' marks unsupported settings
(Appendix~\ref{app:baselines}).
}
\label{tab:missing-modality}
\end{table*}

\section{Experiments}
\label{sec:experiments}
\subsection{Experiment Settings}
\paragraph{Datasets and Evaluation Metrics.} We evaluate GAUGE on six multimodal classification benchmarks,  
including three simulated datasets~\cite{sutter2021generalized}: PolyMNIST, MST, and bimodal CelebA, and two real-world image--tabular datasets: Deep Visual Marketing (DVM)~\cite{huang2022dvm} and UK Biobank (UKBB)~\cite{sudlow2015uk}.
For UKBB, we consider two cardiac disease classification benchmarks: coronary artery disease (CAD) and myocardial infarction (Infarction). We report the area under the curve (AUC) on the two UKBB benchmarks and accuracy on the remaining four benchmarks. Dataset details and missingness protocols are provided in Appendix~\ref{app:datasets}.

\noindent\paragraph{Baselines.}
We compare GAUGE with nine baselines: four recovery-based methods (MultiAE~\cite{ngiam2011multimodal}, MoPoE~\cite{sutter2021generalized}, M3Care~\cite{zhang2022m3care}, and OnlineMAE~\cite{woo2023towards}), four recovery-free methods (ModDrop~\cite{neverova2016moddrop}, MTL~\cite{ma2022missing}, MAP~\cite{lee2023multimodal}, and MUSE~\cite{wu2024muse}), and the dynamic-recovery baseline DyMo$_c$~\cite{du2026inference}. Implementation details are provided in Appendix~\ref{app:impl}.

\begin{figure}[t]
  \centering
    \includegraphics[width=\columnwidth]{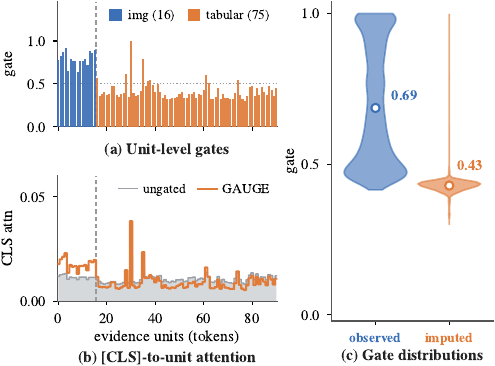}
\caption{Unit-level gating and attention redistribution on the Infarction dataset at missing rate \(\gamma=1\).}
\label{fig:gate}
\end{figure}

\subsection{Main Results \& Analysis}
\paragraph{Quantitative Analysis.} 
GAUGE achieves the best or tied-best performance in 9 of 12 incomplete-input settings reported in Table~\ref{tab:missing-modality}. 
Compared with \emph{recovery-free} and \emph{recovery-based} baselines, GAUGE's stronger performance reflects the \emph{discarding--imputation dilemma}~\cite{du2026inference}: discarding missing modalities may forfeit recoverable task-relevant cues, whereas indiscriminately incorporating imputed content may introduce task-irrelevant noise from unreliable reconstructions. GAUGE mitigates this trade-off through finer-grained evidence modulation. We further compare GAUGE with DyMo$_c$, which uses the same backbone and frozen imputation pipeline but operates at the coarse modality level. GAUGE outperforms DyMo$_c$ in 10 of 12 incomplete-input settings, with larger gains under severe missingness. Specifically, it improves accuracy by $5.90$ percentage points on CelebA $\{\mathrm{T}\}$, with the text modality entirely missing, and yields consistent gains on CAD and Infarction. These gains under severe missingness indicate that, as predictions rely more heavily on reconstructed inputs, unit-wise gating is more effective than modality-level weighting because it down-weights lower-scoring components without uniformly attenuating the recovered modality.

\paragraph{Qualitative Analysis.} Fig.~\ref{fig:gate} visualizes the unit-level gates and final-layer $[\mathtt{CLS}]$-to-unit attention for an Infarction sample at $\gamma=1$ ($16$ observed image units and $75$ imputed tabular units). 
As shown in Fig.~\ref{fig:gate}(a), observed image units generally receive higher gates, whereas the $75$ imputed tabular units receive lower but nonuniform gates.
This within-modality variation demonstrates that GAUGE differentiates encoder-emitted units via their prediction-aware Taylor evidence scores, avoiding the rigidity of modality-level assignment. 
Fig.~\ref{fig:gate}(b) illustrates the resulting attention redistribution. 
Because the gates are injected as key-wise log-biases, lower gates introduce more negative biases, down-weighting lower-scoring units relative to higher-scoring ones.
Across the test set, the mean gate is $0.69$ for observed units and $0.43$ for imputed units (Fig.~\ref{fig:gate}(c)), showing that GAUGE assigns lower gates on average to reconstructed evidence while retaining substantial variation among units within the recovered modality.
Crucially, this nonuniform gating highlights a limitation of modality-level control: a shared modality weight cannot down-weight low-scoring recovered components without simultaneously down-weighting higher-scoring ones from the same modality.

\begin{table}[h]
\centering
\small
\renewcommand{\arraystretch}{1}
\resizebox{\columnwidth}{!}{%
\begin{tabular}{l cccc}
\toprule
\multirow{2}{*}{Method}
 & PolyMNIST & DVM & CAD & Infarction \\
 & $\eta{=}0.8$ & $\gamma{=}1$ & $\gamma{=}1$ & $\gamma{=}1$ \\
\midrule
GAUGE$_{v1}$           & 87.21 & 91.24 & 68.22 & 71.34 \\
GAUGE$_{v2}$ & 97.35 & 94.18 & 69.13 & 68.32 \\
GAUGE$_{v3}$     & 97.23 & 94.27 & 68.95 & 69.52 \\
GAUGE             & \textbf{97.61} & \textbf{94.42} & \textbf{70.41} & \textbf{72.86} \\
\bottomrule
\end{tabular}%
}
\caption{Ablation of fine-grained evidence control under severe missingness.}
\label{tab:ablation-granularity}
\end{table}

\subsection{Ablation Study}
To evaluate the necessity of fine-grained evidence control, we compare
{GAUGE} with three alternatives shown in Table~\ref{tab:ablation-granularity}.
{GAUGE$_{v1}$} performs inference directly on the completed inputs
without fine-tuning or evidence gating.
{GAUGE$_{v2}$} uses the same missingness-aware fine-tuning
as full GAUGE but fixes all gates to one, yielding zero attention bias;
{GAUGE$_{v3}$} shares one gate across all evidence units within
each modality; and
{GAUGE} assigns a separate gate to each evidence unit.
{GAUGE} consistently achieves the best performance across all four benchmarks with multi-unit modalities.
The advantage over modality-level gating (GAUGE$_{v3}$) is
particularly pronounced on clinical benchmarks with many evidence units
in the recovered tabular modality, reaching \(3.34\) AUC points on
Infarction. Notably, modality-level gating can underperform {GAUGE$_{v1}$}, as a
shared modality gate cannot suppress low-scoring units without simultaneously
down-weighting high-scoring ones from the same modality; unit-level gating
mitigates this limitation by assigning a separate gate to each evidence unit.
Further ablations on evidence-unit scoring, score normalization,
and evidence-unit granularity are provided in Appendix~\ref{app:supp}.

\subsection{Additional Analysis}
\label{subsec:efficiency}
\paragraph{Efficiency Analysis.}
We compare GAUGE with exact counterfactual scoring across four configurations with different evidence-unit counts \(N\). Table~\ref{tab:cost} shows that exact scoring time increases markedly with \(N\), whereas GAUGE remains within \(3.11\)--\(3.60\)\,ms. At \(N{=}91\), GAUGE reduces the per-sample scoring time from \(96.43\)\,ms to \(3.11\)\,ms, yielding a \(31.01\times\) speedup. These results show that GAUGE avoids per-unit intervention overhead and supports scalable fine-grained evidence scoring. Detailed timing protocol is in Appendix~\ref{app:cost}.

\begin{table}[h]
\centering
\renewcommand{\arraystretch}{1}
\resizebox{\columnwidth}{!}{%
\begin{tabular}{lcccc}
\toprule
Dataset & $N$ & GAUGE (ms) & Exact (ms) & Speedup \\
\midrule
MST       &  3 & 3.15 &  4.51 &  1.43$\times$ \\
PolyMNIST & 20 & 3.60 & 26.30 &  7.31$\times$ \\
DVM       & 33 & 3.18 & 39.06 & 12.28$\times$ \\
CAD       & 91 & 3.11 & 96.43 & 31.01$\times$ \\
\bottomrule
\end{tabular}}
\caption{Per-sample scoring time for GAUGE and exact counterfactual scoring on a Quadro RTX~8000 with batch size \(1\).}
\label{tab:cost}
\end{table}

\begin{figure}[t]
\centering
\includegraphics[width=0.93\columnwidth]{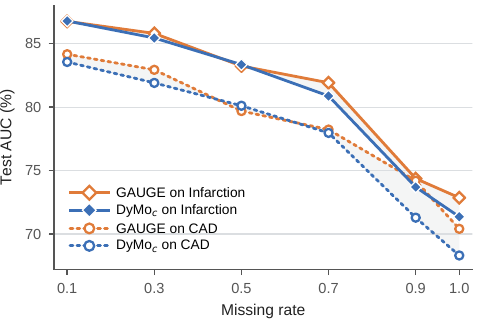}
\caption{Test AUC (\%) on CAD and Infarction under varying tabular
missing rates \(\gamma\).}
\label{fig:missing_rate}
\end{figure}

\paragraph{Performance under Increasing Missingness.}
To evaluate robustness across missing rates, we compare GAUGE and DyMo$_c$ on CAD and Infarction. As shown in Fig.~\ref{fig:missing_rate}, GAUGE outperforms DyMo$_c$ in most scenarios, demonstrating consistent adaptability from mild to complete tabular absence. 
Crucially, GAUGE's advantage becomes increasingly pronounced as missingness rises. 
When missingness is mild ($\gamma \le 0.5$) and observed features dominate, both methods perform comparably. However, once imputation becomes dominant ($\gamma \ge 0.7$), GAUGE generally exhibits a larger advantage over DyMo$_c$. At extreme missingness ($\gamma = 1.0$), GAUGE surpasses DyMo$_c$ by $2.10$ and $1.50$ percentage points in AUC on CAD and Infarction, respectively. This further demonstrates that fine-grained evidence control becomes increasingly important when models rely heavily on reconstructed content.

\section{Conclusion and Limitations}
In this paper, we introduced GAUGE, a principled and scalable framework for granularity-adaptive evidence control in incomplete multimodal classification. 
By formulating a unit-wise counterfactual gating mechanism driven by pass-efficient Taylor evidence scoring, GAUGE effectively transitions multimodal evidence modulation from coarse modality-level aggregation to the finest granularity exposed by Transformer encoders. 
Extensive experiments across diverse benchmarks demonstrate that GAUGE consistently outperforms state-of-the-art baselines, exhibiting exceptional robustness under severe modality incompleteness. However, GAUGE currently relies on the pre-trained backbones and frozen imputers inherited from existing baselines. Future work will explore extending this pass-efficient counterfactual gating paradigm to broader multimodal reasoning and generation architectures.

\section*{Acknowledgments}
The work was supported by the Australian Research Council (ARC) under Laureate project FL190100149. This research has been conducted using the UK Biobank Resource under Project Number 105141.

\bibliography{ref}

@inproceedings{zhang2025multimodal,
  title={Multimodal inverse attention network with intrinsic discriminant feature exploitation for fake news detection},
  author={Zhang, Tianlin and Yu, En and Shao, Yi and Sun, Jiande},
  booktitle={Proceedings of the Thirty-Fourth International Joint Conference on Artificial Intelligence},
  pages={7940--7948},
  year={2025}
}

@inproceedings{yang2025adapting,
title={Adapting Multi-modal Large Language Model to Concept Drift From Pre-training Onwards},
author={Xiaoyu Yang and Jie Lu and En Yu},
booktitle={The Thirteenth International Conference on Learning Representations},
year={2025},
url={https://openreview.net/forum?id=b20VK2GnSs}
}

@inproceedings{li2025causality,
  title={Causality-aligned Prompt Learning via Diffusion-based Counterfactual Generation},
  author={Li, Xinshu and Wang, Ruoyu and Gao, Erdun and Gong, Mingming and Yao, Lina},
  booktitle={Proceedings of the 33rd ACM International Conference on Multimedia},
  pages={5208--5217},
  year={2025}
}

@article{wang2018task,
  title={Task-dependent and query-dependent subspace learning for cross-modal retrieval},
  author={Wang, Li and Zhu, Lei and Yu, En and Sun, Jiande and Zhang, Huaxiang},
  journal={IEEE Access},
  volume={6},
  pages={27091--27102},
  year={2018},
  publisher={IEEE}
}

@inproceedings{yu2026generalized,
  title={Generalized incremental learning under concept drift across evolving data streams},
  author={Yu, En and Lu, Jie and Zhang, Guangquan},
  booktitle={Proceedings of the ACM Web Conference 2026},
  pages={3905--3916},
  year={2026}
}

@article{liang2022foundations,
  title={Foundations \& trends in multimodal machine learning: Principles, challenges, and open questions},
  author={Liang, Paul Pu and Zadeh, Amir and Morency, Louis-Philippe},
  journal={ACM computing surveys},
  volume={56},
  number={10},
  pages={1--42},
  year={2024}
}

@article{neverova2016moddrop,
  title={{ModDrop}: adaptive multi-modal gesture recognition},
  author={Neverova, Natalia and Wolf, Christian and Taylor, Graham and Nebout, Florian},
  journal={IEEE Transactions on Pattern Analysis and Machine Intelligence},
  volume={38},
  number={8},
  pages={1692--1706},
  year={2016},
  publisher={IEEE}
}

@inproceedings{ma2021smil,
  title={{SMIL}: Multimodal learning with severely missing modality},
  author={Ma, Mengmeng and Ren, Jian and Zhao, Long and Tulyakov, Sergey and Wu, Cathy and Peng, Xi},
  booktitle={Proceedings of the AAAI conference on artificial intelligence},
  volume={35},
  number={3},
  pages={2302--2310},
  year={2021}
}

@inproceedings{ma2022missing,
  title={Are multimodal transformers robust to missing modality?},
  author={Ma, Mengmeng and Ren, Jian and Zhao, Long and Testuggine, Davide and Peng, Xi},
  booktitle={Proceedings of the IEEE/CVF conference on computer vision and pattern recognition},
  pages={18177--18186},
  year={2022}
}

@inproceedings{lee2023multimodal,
  title={Multimodal prompting with missing modalities for visual recognition},
  author={Lee, Yi-Lun and Tsai, Yi-Hsuan and Chiu, Wei-Chen and Lee, Chen-Yu},
  booktitle={Proceedings of the IEEE/CVF Conference on Computer Vision and Pattern Recognition},
  pages={14943--14952},
  year={2023}
}

@inproceedings{wu2024muse,
  title={Multimodal patient representation learning with missing modalities and labels},
  author={Wu, Zhenbang and Dadu, Anant and Tustison, Nicholas and Avants, Brian and Nalls, Michael and Sun, Jimeng and Faghri, Faraz},
  booktitle={International Conference on Learning Representations},
  year={2024}
}

@inproceedings{sutter2021generalized,
  title={Generalized Multimodal {ELBO}},
  author={Sutter, Thomas M. and Daunhawer, Imant and Vogt, Julia E.},
  booktitle={International Conference on Learning Representations},
  year={2021}
}

@inproceedings{zhang2022m3care,
  title={{M3Care}: Learning with missing modalities in multimodal healthcare data},
  author={Zhang, Chaohe and Chu, Xu and Ma, Liantao and Zhu, Yinghao and Wang, Yasha and Wang, Jiangtao and Zhao, Junfeng},
  booktitle={Proceedings of the 28th ACM SIGKDD conference on knowledge discovery and data mining},
  pages={2418--2428},
  year={2022}
}

@article{wang2023incomplete,
  title={Incomplete multimodality-diffused emotion recognition},
  author={Wang, Yuanzhi and Li, Yong and Cui, Zhen},
  journal={Advances in Neural Information Processing Systems},
  volume={36},
  pages={17117--17128},
  year={2023}
}

@inproceedings{yao2024drfuse,
  title={{DrFuse}: Learning disentangled representation for clinical multi-modal fusion with missing modality and modal inconsistency},
  author={Yao, Wenfang and Yin, Kejing and Cheung, William K and Liu, Jia and Qin, Jing},
  booktitle={Proceedings of the AAAI conference on artificial intelligence},
  volume={38},
  number={15},
  pages={16416--16424},
  year={2024}
}

@inproceedings{du2024tip,
  title={{TIP}: Tabular-image pre-training for multimodal classification with incomplete data},
  author={Du, Siyi and Zheng, Shaoming and Wang, Yinsong and Bai, Wenjia and O'Regan, Declan P. and Qin, Chen},
  booktitle={European Conference on Computer Vision},
  pages={478--496},
  year={2024},
  publisher={Springer}
}

@inproceedings{du2026inference,
  title     = {Inference-Time Dynamic Modality Selection for Incomplete Multimodal Classification},
  author    = {Du, Siyi and Luo, Xinzhe and O'Regan, Declan P. and Qin, Chen},
  booktitle = {International Conference on Learning Representations (ICLR)},
  year      = {2026}
}

@inproceedings{xue2023dynamic,
  title={Dynamic Multimodal Fusion},
  author={Xue, Zihui and Marculescu, Radu},
  booktitle={2023 IEEE/CVF Conference on Computer Vision and Pattern Recognition Workshops (CVPRW)},
  pages={2575--2584},
  year={2023},
  organization={IEEE}
}

@inproceedings{shrikumar2017deeplift,
  title={Learning important features through propagating activation differences},
  author={Shrikumar, Avanti and Greenside, Peyton and Kundaje, Anshul},
  booktitle={International conference on machine learning},
  pages={3145--3153},
  year={2017},
  organization={PMLR}
}

@inproceedings{wang2024gmd,
  title={Gradient-guided modality decoupling for missing-modality robustness},
  author={Wang, Hao and Luo, Shengda and Hu, Guosheng and Zhang, Jianguo},
  booktitle={Proceedings of the AAAI Conference on Artificial Intelligence},
  volume={38},
  number={14},
  pages={15483--15491},
  year={2024}
}

@inproceedings{zhang2023qmf,
  title={Provable dynamic fusion for low-quality multimodal data},
  author={Zhang, Qingyang and Wu, Haitao and Zhang, Changqing and Hu, Qinghua and Fu, Huazhu and Zhou, Joey Tianyi and Peng, Xi},
  booktitle={International conference on machine learning},
  pages={41753--41769},
  year={2023},
  organization={PMLR}
}

@inproceedings{cao2024pdf,
  title={Predictive Dynamic Fusion},
  author={Cao, Bing and Xia, Yinan and Ding, Yi and Zhang, Changqing and Hu, Qinghua},
  booktitle={International Conference on Machine Learning},
  pages={5608--5628},
  year={2024},
  organization={PMLR}
}

@inproceedings{sundararajan2017axiomatic,
  title={Axiomatic attribution for deep networks},
  author={Sundararajan, Mukund and Taly, Ankur and Yan, Qiqi},
  booktitle={International conference on machine learning},
  pages={3319--3328},
  year={2017},
  organization={PMLR}
}

@article{vaswani2017attention,
  title={Attention is all you need},
  author={Vaswani, Ashish and Shazeer, Noam and Parmar, Niki and Uszkoreit, Jakob and Jones, Llion and Gomez, Aidan N and Kaiser, {\L}ukasz and Polosukhin, Illia},
  journal={Advances in neural information processing systems},
  volume={30},
  year={2017}
}

@inproceedings{deng2021unified,
  title={A unified {T}aylor framework for revisiting attribution methods},
  author={Deng, Huiqi and Zou, Na and Du, Mengnan and Chen, Weifu and Feng, Guocan and Hu, Xia},
  booktitle={Proceedings of the AAAI Conference on Artificial Intelligence},
  volume={35},
  number={13},
  pages={11462--11469},
  year={2021}
}

@inproceedings{huang2022dvm,
  title={{DVM-CAR}: A large-scale automotive dataset for visual marketing research and applications},
  author={Huang, Jingmin and Chen, Bowei and Luo, Lan and Yue, Shigang and Ounis, Iadh},
  booktitle={2022 IEEE International Conference on Big Data (Big Data)},
  pages={4140--4147},
  year={2022},
  organization={IEEE}
}

@article{sudlow2015uk,
title={{UK Biobank}: An Open Access Resource for Identifying the Causes of a Wide Range of Complex Diseases of Middle and Old Age},
  author={Sudlow, Cathie and Gallacher, John and Allen, Naomi and Beral, Valerie and Burton, Paul and Danesh, John and Downey, Paul and Elliott, Paul and Green, Jane and Landray, Martin and others},
  journal={PLOS Medicine},
  volume={12},
  number={3},
  pages={e1001779},
  year={2015},
  publisher={Public Library of Science}
}

@inproceedings{ngiam2011multimodal,
  title={Multimodal Deep Learning},
  author={Ngiam, Jiquan and Khosla, Aditya and Kim, Mingyu and Nam, Juhan and Lee, Honglak and Ng, Andrew Y.},
  booktitle={Proceedings of the 28th International Conference on Machine Learning},
  pages={689--696},
  year={2011}
}

@inproceedings{woo2023towards,
  title={Towards good practices for missing modality robust action recognition},
  author={Woo, Sangmin and Lee, Sumin and Park, Yeonju and Nugroho, Muhammad Adi and Kim, Changick},
  booktitle={Proceedings of the AAAI Conference on Artificial Intelligence},
  volume={37},
  number={3},
  pages={2776--2784},
  year={2023}
}

@inproceedings{sun2024redcore,
  title={{RedCore}: Relative advantage aware cross-modal representation learning for missing modalities with imbalanced missing rates},
  author={Sun, Jun and Zhang, Xinxin and Han, Shoukang and Ruan, Yu-Ping and Li, Taihao},
  booktitle={Proceedings of the AAAI Conference on Artificial Intelligence},
  volume={38},
  number={13},
  pages={15173--15182},
  year={2024}
}

@inproceedings{
ancona2018unified,
title={Towards better understanding of gradient-based attribution methods for Deep Neural Networks},
author={Ancona, Marco and Ceolini, Enea and {\"O}ztireli, Cengiz and Gross, Markus},
booktitle={International Conference on Learning Representations},
year={2018},
url={https://openreview.net/forum?id=Sy21R9JAW},
}

@inproceedings{zhang2025synergistic,
  title={Synergistic prompting for robust visual recognition with missing modalities},
  author={Zhang, Zhihui and Dai, Luanyuan and Lin, Qika and Diao, Yunfeng and Jin, Guangyin and Guo, Yufei and Zhang, Jing and Hao, Xiaoshuai},
  booktitle={Proceedings of the IEEE/CVF International Conference on Computer Vision},
  pages={1881--1890},
  year={2025}
}

@inproceedings{lang2025ragpt,
  title={Retrieval-augmented dynamic prompt tuning for incomplete multimodal learning},
  author={Lang, Jian and Cheng, Zhangtao and Zhong, Ting and Zhou, Fan},
  booktitle={Proceedings of the AAAI Conference on Artificial Intelligence},
  volume={39},
  number={17},
  pages={18035--18043},
  year={2025}
}

@inproceedings{liu2025imdr,
  title={Incomplete modality disentangled representation for ophthalmic disease grading and diagnosis},
  author={Liu, Chengzhi and Huang, Zile and Chen, Zhe and Tang, Feilong and Tian, Yu and Xu, Zhongxing and Luo, Zihong and Zheng, Yalin and Meng, Yanda},
  booktitle={Proceedings of the AAAI Conference on Artificial Intelligence},
  volume={39},
  number={5},
  pages={5361--5369},
  year={2025}
}

@inproceedings{chen2026smcir,
  title={Sample-specific modality diagnosis and cross-modal enhancement for incomplete multimodal representations},
  author={Chen, Junsong and Liu, Jiyuan and Liu, Suyuan and Zhang, Wei and Li, Ao and Zhu, En and Liu, Xinwang},
  booktitle={Proceedings of the AAAI Conference on Artificial Intelligence},
  volume={40},
  number={24},
  pages={20154--20162},
  year={2026}
}

@inproceedings{zhao2026dualstage,
  title={Tackling dual-stage missing modalities in brain tumor segmentation via robust modality reconstruction and prompt-guided modality adaptation},
  author={Zhao, Yunpeng and Chen, Cheng and Pang, Qing You and Fu, Yibing and Li, Quanzheng and Tang, Carol and Ang, Beng Ti and Jin, Yueming},
  booktitle={Proceedings of the AAAI Conference on Artificial Intelligence},
  volume={40},
  number={16},
  pages={13314--13322},
  year={2026}
}

@inproceedings{huang2026recap,
  title={Recovering coherent affective patterns: Addressing modality missing in multimodal sentiment analysis},
  author={Huang, Huiting and Gong, Tieliang and He, Kai and Wen, Wen and Zhang, Weizhan and Feng, Mengling},
  booktitle={Proceedings of the AAAI Conference on Artificial Intelligence},
  volume={40},
  number={26},
  pages={21957--21965},
  year={2026}
}

@inproceedings{park2025resilient,
  title={Resilient sensor fusion under adverse sensor failures via multi-modal expert fusion},
  author={Park, Konyul and Kim, Yecheol and Kim, Daehun and Choi, Jun Won},
  booktitle={Proceedings of the IEEE/CVF Conference on Computer Vision and Pattern Recognition},
  pages={6720--6729},
  year={2025}
}

@article{xu2023multimodal,
  title={Multimodal learning with transformers: A survey},
  author={Xu, Peng and Zhu, Xiatian and Clifton, David A},
  journal={IEEE Transactions on Pattern Analysis and Machine Intelligence},
  volume={45},
  number={10},
  pages={12113--12132},
  year={2023},
  publisher={IEEE}
}

@inproceedings{fong2017interpretable,
  title={Interpretable explanations of black boxes by meaningful perturbation},
  author={Fong, Ruth C and Vedaldi, Andrea},
  booktitle={Proceedings of the IEEE international conference on computer vision},
  pages={3429--3437},
  year={2017}
}

@InProceedings{zhu2026decoupled,
    author    = {Zhu, Aoqiang and Hu, Min and Xing, Yan and Tang, Yiming},
    title     = {Decoupled Sub-Feature Uncertainty Modeling for Robust Multimodal Representation Learning},
    booktitle = {Proceedings of the IEEE/CVF Conference on Computer Vision and Pattern Recognition (CVPR) Findings},
    month     = {June},
    year      = {2026},
    pages={6921--6931}
}

@inproceedings{fan2026enhancing,
  title={Enhancing interpretability for vision models via {S}hapley value optimization},
  author={Fan, Kanglong and Yang, Yunqiao and Ma, Chen},
  booktitle={Proceedings of the AAAI Conference on Artificial Intelligence},
  volume={40},
  number={5},
  pages={3786--3794},
  year={2026}
}

@inproceedings{dai2025unbiased,
  title     = {Unbiased Missing-modality Multimodal Learning},
  author    = {Dai, Ruiting and Li, Chenxi and Yan, Yandong and Mo, Lisi and Qin, Ke and He, Tao},
  booktitle = {Proceedings of the IEEE/CVF International Conference on Computer Vision (ICCV)},
  pages     = {24507--24517},
  year      = {2025},
  doi       = {10.1109/ICCV51701.2025.02272}
}
\clearpage
\newpage
\appendix
\setcounter{secnumdepth}{2}
\section{Theoretical Analysis}
\label{app:te}

This appendix analyzes the Taylor evidence score used in GAUGE. It first defines the exact unit-level counterfactual effect and derives a first-order approximation with a remainder bound. It then examines the non-canceling aggregation and explains the choice of a logit-based saliency objective.
\subsection{Intervention Setup and Notation}
\label{app:te-setup}
We first establish the intervention notation used to define the exact counterfactual effect, its signed first-order
approximation, and the resulting Taylor evidence score. For sample \(i\), replacing evidence unit \(e_{i,j}\) in \(\tilde{T}_i = [\,e_{[\mathtt{CLS}]};e_{i,1},\ldots,e_{i,N}\,]\) with a reference representation \(e'_{i,j}\) yields \(\tilde{T}_i^{(j\leftarrow e'_{i,j})}\), with replacement direction \(d_{i,j}=e'_{i,j}-e_{i,j}\).
Throughout this work, we use the zero reference \(e'_{i,j}=\mathbf{0}\), while the \([\mathtt{CLS}]\) token is never replaced.

\subsection{Exact Unit-Level Counterfactual Effect}
\label{app:direct-te}
The saliency objective is \(\mathcal{L}_{\mathrm{sal}}(\tilde{T}_i)=-z_{i,\hat{y}_i}\), the negative logit of the ungated predicted class. Here $\hat{y}_i$ is determined once from the unreplaced sequence $\tilde{T}_i$ and is held constant under both evidence-unit replacement and differentiation: we do not differentiate through the $\arg\max$, and if the predicted class changes after replacement, $\mathcal{L}_{\mathrm{sal}}(\cdot)$ still denotes the negative logit of the original ungated predicted class $\hat{y}_i$. With $\hat{y}_i$ fixed, the exact counterfactual effect of unit $j$ is defined as:
\[
\mathrm{CEff}_{i,j}^{\mathrm{exact}}
=\mathcal{L}_{\mathrm{sal}}\!\big(\tilde{T}_i^{(j\leftarrow e'_{i,j})}\big)
-\mathcal{L}_{\mathrm{sal}}(\tilde{T}_i),
\]
and its magnitude $|\mathrm{CEff}_{i,j}^{\mathrm{exact}}|$ measures the response change induced by replacing unit $j$ with its reference representation. Computing this for all units requires one counterfactual forward pass per unit, i.e., $N$ additional forward passes after the ungated forward pass. GAUGE instead constructs first-order local surrogates for all \(N\) effects using a single ungated forward--backward scoring pass.

\subsection{First-Order Counterfactual Approximation}
\label{app:first-order-taylor}
Because evaluating the exact effect separately for every evidence unit is computationally expensive, we derive a local first-order surrogate
whose unit-wise values can be obtained from a single backward pass. Fixing all evidence units except $e_{i,j}$, regard the saliency objective as a function of this unit alone:
\begin{equation}
\psi_{i,j}(v)
=\mathcal{L}_{\mathrm{sal}}\!\big(\tilde{T}_i^{(j\leftarrow v)}\big).
\label{eq:app-psi}
\end{equation}
With $\hat{y}_i$ fixed, $\psi_{i,j}$ is a scalar function of $v$, with $\psi_{i,j}(e_{i,j})=\mathcal{L}_{\mathrm{sal}}(\tilde{T}_i)$ and
$\psi_{i,j}(e'_{i,j})
=\mathcal{L}_{\mathrm{sal}}(\tilde{T}_i^{(j\leftarrow e'_{i,j})})$, so that
$\mathrm{CEff}_{i,j}^{\mathrm{exact}}
=\psi_{i,j}(e'_{i,j})-\psi_{i,j}(e_{i,j})$.
A first-order Taylor expansion of $\psi_{i,j}$ at $e_{i,j}$ gives
\begin{equation}
\mathrm{CEff}_{i,j}^{\mathrm{exact}}
\approx
a_{i,j}
:=\big\langle
\nabla_{e_{i,j}}\mathcal{L}_{\mathrm{sal}}(\tilde{T}_i),\, d_{i,j}
\big\rangle,
\label{eq:app-lsal-first-order}
\end{equation}
which is the first-order term of
Eq.~\eqref{eq:taylor_expansion} in the main text; we call $a_{i,j}$ the \emph{signed first-order term}. This formulation makes all unit-wise first-order terms available from a single backward pass. The following analysis first characterizes its approximation error and then addresses the cancellation that can arise when channel-wise contributions are summed.

\subsection{Taylor Remainder and Approximation Error}
\label{app:te-thm}
We now characterize the approximation error of the signed first-order term relative to the exact counterfactual effect under a local smoothness condition.
\begin{theorem}[First-order Counterfactual Approximation]
\label{thm:te}
Let
\[
\Omega_{i,j}=\{\,e_{i,j}+\alpha\, d_{i,j}:\alpha\in[0,1]\,\}
\]
be the line segment between $e_{i,j}$ and $e'_{i,j}$. Assume that $\psi_{i,j}$ is twice continuously differentiable on a neighborhood of
$\Omega_{i,j}$ and that the unit-wise Hessian is uniformly bounded on this segment:
\[
\sup_{v\in\Omega_{i,j}}
\big\|\nabla^2\psi_{i,j}(v)\big\|_{\mathrm{op}}\le H_{i,j}.
\]
Then the remainder
$R_{i,j}:=\mathrm{CEff}_{i,j}^{\mathrm{exact}}-a_{i,j}$ satisfies
\begin{equation}
\mathrm{CEff}_{i,j}^{\mathrm{exact}}
=a_{i,j}+R_{i,j},
\qquad
|R_{i,j}|\le\frac{H_{i,j}}{2}\,\|d_{i,j}\|_2^2,
\label{eq:te-remainder}
\end{equation}
recovering Eq.~\eqref{eq:taylor_expansion} of the main text with an explicit remainder bound. Consequently,
\begin{equation}
\Big|\,\big|\mathrm{CEff}_{i,j}^{\mathrm{exact}}\big|-|a_{i,j}|\,\Big|
\le
\frac{H_{i,j}}{2}\,\|d_{i,j}\|_2^2,
\label{eq:te-abs}
\end{equation}
i.e., $|a_{i,j}|$ approximates the magnitude of the exact counterfactual effect up to a second-order error.
\end{theorem}

\begin{remark}[ReLU and smoothness]
Networks with ReLU are not twice differentiable at activation boundaries. The stated remainder bound therefore applies when the replacement segment remains within a single differentiable activation region. If the segment crosses activation boundaries, the \(C^2\) assumption of Theorem~\ref{thm:te} need not hold. For smooth activations such as GELU or SiLU, a locally bounded Hessian can instead be assumed on the finite replacement segment. The theorem requires only this local bound, not a global Hessian bound.
\end{remark}

\begin{proof}
Let $\psi_{i,j}$ be defined as in
Eq.~\eqref{eq:app-psi}. By Taylor's theorem with Lagrange remainder, there exists a point $\xi_{i,j}\in\Omega_{i,j}$ such that
\begin{equation}
\begin{aligned}
\psi_{i,j}(e'_{i,j})
&=
\psi_{i,j}(e_{i,j})
+
\big\langle \nabla\psi_{i,j}(e_{i,j}),\, d_{i,j} \big\rangle \\
&\quad+
\frac{1}{2}\,
d_{i,j}^\top
\nabla^2\psi_{i,j}(\xi_{i,j})\,
d_{i,j}.
\end{aligned}
\end{equation}
Since
$\nabla\psi_{i,j}(e_{i,j})
=\nabla_{e_{i,j}}\mathcal{L}_{\mathrm{sal}}(\tilde{T}_i)$
and
$\mathrm{CEff}_{i,j}^{\mathrm{exact}}
=\psi_{i,j}(e'_{i,j})-\psi_{i,j}(e_{i,j})$,
subtracting $\psi_{i,j}(e_{i,j})$ from both sides identifies the remainder as
\begin{equation}
R_{i,j}
=\mathrm{CEff}_{i,j}^{\mathrm{exact}}-a_{i,j}
=
\frac{1}{2}\,
d_{i,j}^\top
\nabla^2\psi_{i,j}(\xi_{i,j})\,
d_{i,j}.
\end{equation}
The Hessian bound implies
\begin{equation}
|R_{i,j}|
\le
\frac{1}{2}
\big\|\nabla^2\psi_{i,j}(\xi_{i,j})\big\|_{\mathrm{op}}
\|d_{i,j}\|_2^2
\le
\frac{H_{i,j}}{2}\,
\|d_{i,j}\|_2^2.
\end{equation}
Finally, since $|\mathrm{CEff}_{i,j}^{\mathrm{exact}}|=|a_{i,j}+R_{i,j}|$, the reverse triangle inequality
$\big|\,|a_{i,j}+R_{i,j}|-|a_{i,j}|\,\big|\le|R_{i,j}|$ proves Eq.~\eqref{eq:te-abs}.
\end{proof}
The theorem therefore identifies the regime in which the signed first-order term is a faithful local surrogate: its approximation
error is controlled by the local curvature and grows quadratically with the replacement distance.
\subsection{Non-Canceling Taylor Evidence Scoring}
\label{app:te-l1}
The magnitude of the signed first-order term can be reduced by cancellation among channel-wise contributions with opposite signs. We therefore motivate GAUGE's non-canceling aggregation and establish its relation to the exact counterfactual effect. The signed first-order term can be written channel-wise as
\[
a_{i,j}
=\sum_{c=1}^{C}
\frac{\partial \mathcal{L}_{\mathrm{sal}}(\tilde{T}_i)}
{\partial e_{i,j,c}}\,
d_{i,j,c}.
\]
GAUGE instead sums the absolute channel-wise contributions. This avoids sign cancellation, although the resulting score may exceed the magnitude of the signed first-order term, \(|a_{i,j}|\). 
Specifically,
\begin{equation}
\begin{aligned}
\mathrm{TE}_{i,j}
&=
\operatorname{sg}\!\left(
\sum_{c=1}^{C}
\left|
\frac{\partial \mathcal{L}_{\mathrm{sal}}(\tilde{T}_i)}
{\partial e_{i,j,c}}\,
d_{i,j,c}
\right|
\right) \\
&=
\operatorname{sg}\!\left(
\big\|
\nabla_{e_{i,j}}\mathcal{L}_{\mathrm{sal}}(\tilde{T}_i)
\odot d_{i,j}
\big\|_1
\right).
\end{aligned}
\label{eq:app-te-l1-score}
\end{equation}
The stop-gradient operator does not change the numerical value of $\mathrm{TE}_{i,j}$ and therefore does not affect the inequalities below; it only prevents the task loss from differentiating through the gradient-based score.

\begin{corollary}
\label{cor:upper}
\normalfont
Because \(\operatorname{sg}(\cdot)\) does not alter the forward value, the triangle inequality yields
\[
|a_{i,j}|\le \mathrm{TE}_{i,j}.
\]
Thus, $\mathrm{TE}_{i,j}$ upper-bounds the magnitude of the signed first-order term and serves as a non-canceling channel-wise magnitude surrogate. Combining this with Theorem~\ref{thm:te} gives
\begin{equation}
\big|\mathrm{CEff}_{i,j}^{\mathrm{exact}}\big|
\le
\mathrm{TE}_{i,j}
+
\frac{H_{i,j}}{2}\,\|d_{i,j}\|_2^2,
\label{eq:te-l1-upper}
\end{equation}
which is Eq.~\eqref{eq:te_ceff_bound} of the main text.
\end{corollary}
The Hessian term appears only in the theoretical error bound; GAUGE neither estimates \(H_{i,j}\) nor uses it during training.

\subsection{Choice of Saliency Objective}
\label{app:logit-vs-ce}
The quality of the resulting evidence score also depends on the saliency objective. We therefore compare the predicted-class logit with negative log-likelihood and show why the former avoids softmax-induced gradient saturation. One alternative is to use predicted-class cross-entropy, equivalently the negative log-likelihood (NLL), as the saliency target:
\begin{equation}
\mathcal{L}_{\mathrm{NLL}}(\tilde{T}_i)
=
-\log p_i(\hat{y}_i).
\end{equation}
Its gradient with respect to the logits is
\begin{equation}
\nabla_{z_i}\mathcal{L}_{\mathrm{NLL}}
=
p_i-\mathbf{1}_{\hat{y}_i}.
\end{equation}
Here, $\mathbf{1}_{\hat{y}_i}$ is the one-hot vector of the ungated predicted class. By the chain rule, the evidence-unit gradient is
\begin{equation}
\nabla_{e_{i,j}}\mathcal{L}_{\mathrm{NLL}}
=
\left(\frac{\partial z_i}{\partial e_{i,j}}\right)^{\!\top}
\left(p_i-\mathbf{1}_{\hat{y}_i}\right).
\end{equation}
For a confident model, \(p_i\approx\mathbf{1}_{\hat{y}_i}\). When the logits--evidence-unit Jacobian is bounded, the factor
\(p_i-\mathbf{1}_{\hat{y}_i}\) can make the evidence-unit gradient very small.

NLL-based saliency may therefore underestimate the contribution of informative evidence units due to softmax saturation.
In contrast, the gradient of \(\mathcal{L}_{\mathrm{sal}}=-z_{i,\hat{y}_i}\) with respect to the logits is the constant
\(-\mathbf{1}_{\hat{y}_i}\), bypassing the softmax and avoiding saturation. This is consistent with attribution methods that operate at the pre-softmax level~\cite{shrikumar2017deeplift,sundararajan2017axiomatic}. Accordingly, GAUGE combines a signed Taylor term for pass-efficient local approximation, \(\ell_1\) aggregation to prevent channel-wise cancellation, and the predicted-class logit to avoid softmax saturation.

\section{Experimental Configuration}
\subsection{Datasets and Input Configurations}
\noindent\textbf{Datasets.}
\label{app:datasets}
Table~\ref{tab:dataset-statistics} summarizes the six multimodal classification benchmarks derived from five datasets. PolyMNIST, MST (MNIST--SVHN--Text), and bimodal CelebA~\cite{sutter2021generalized} are three simulated multimodal benchmarks. PolyMNIST contains five image views of the same digit. MST consists of three heterogeneous modalities: MNIST (M), SVHN (S), and synthetic text (T). CelebA pairs face images (I) with attribute-based text descriptions (T). We further evaluate GAUGE on DVM and two cardiac classification benchmarks derived from UK Biobank. DVM~\cite{huang2022dvm} combines vehicle images with \(17\) tabular features, including \(4\) categorical and \(13\) continuous features. The coronary artery disease (CAD) and myocardial infarction (Infarction) benchmarks~\cite{sudlow2015uk} each combine cardiac MR images with \(75\) disease-related tabular features, including \(26\) categorical and \(49\) continuous features. Because both diseases have low prevalence, class-balanced subsets are used for training. The dataset splits follow DyMo~\cite{du2026inference}.
\begin{table*}[!t]
\centering
\small
\setlength{\tabcolsep}{5pt}
\begin{tabular}{l l r r r r}
\toprule
Dataset/Task & Modalities & Train & Validation & Test & Classes \\
\midrule
PolyMNIST  & Five image views       & 60,000    & 3,000  & 7,000   & 10 \\
MST        & MNIST, SVHN, text      & 1,121,360 & 60,000 & 140,000 & 10 \\
CelebA     & Image, text            & 162,770   & 19,962 & 19,867  & 2 \\
DVM        & Image, tabular         & 70,565    & 17,642 & 88,207  & 283 \\
CAD        & MR image, tabular      & 3,482     & 6,510  & 3,617   & 2 \\
Infarction & MR image, tabular      & 1,552     & 6,510  & 3,617   & 2 \\
\bottomrule
\end{tabular}
\caption{Dataset statistics for the six benchmarks, using the splits of \citet{du2026inference}. CAD and Infarction are constructed from UK Biobank.}
\label{tab:dataset-statistics}
\end{table*}

\noindent\textbf{Missingness Protocols.}
\label{app:feature-missing}
Following DyMo~\cite{du2026inference}, we use the same mask-generation procedure and benchmark-specific missingness protocols. For PolyMNIST, each sample randomly misses \(\eta\times100\%\) of its five image modalities, and Table~\ref{tab:missing-modality} reports \(\eta\in\{0,0.6,0.8\}\). For MST, we report the complete-input setting (\(\emptyset\)) and the
missing-modality subsets \(\{\mathrm{S},\mathrm{T}\}\) and \(\{\mathrm{M},\mathrm{T}\}\). For CelebA, we report \(\emptyset\), \(\{\mathrm{I}\}\), and \(\{\mathrm{T}\}\). For both benchmarks, each set denotes the modalities removed from the input. For DVM, CAD, and Infarction, the imaging modality remains observed, while each sample randomly misses \(\gamma\times100\%\) of its tabular features. Table~\ref{tab:missing-modality} reports \(\gamma\in\{0,0.9,1\}\) for all three benchmarks. For CAD and Infarction, Fig.~\ref{fig:missing_rate} additionally includes \(\gamma\in\{0.1,0.3,0.5,0.7\}\). Here, \(\eta=0\) or \(\gamma=0\) corresponds to complete input, and \(\gamma=1\) indicates that the entire tabular modality is missing. To connect the feature-level protocol above with the formulation in the main text, let \(t\) denote the tabular modality. For \(0<\gamma<1\), \(t\) remains partially observed and therefore belongs to \(\mathcal{O}_i\). Its missing entries are reconstructed by the frozen TIP imputer used in the DyMo pipeline:
\[
\tilde{x}_i^{(t)}
=\Gamma_t\!\big(
\{x_i^{(m)}\}_{m\in\mathcal{O}_i\setminus\{t\}},\,
x_{i,\mathrm{obs}}^{(t)}
\big),
\]
where \(x_{i,\mathrm{obs}}^{(t)}\) denotes the observed tabular entries. The completed modality \(\bar{x}_i^{(t)}\) retains the observed entries and fills the missing positions with the corresponding imputed values. When \(\gamma=1\), the tabular modality belongs to \(\mathcal{U}_i\), reducing to whole-modality imputation. In all cases, GAUGE operates only on the completed input \(\bar{X}_i\) produced by the frozen DyMo/TIP pipeline.

\subsection{Baselines}
\label{app:baselines}
We compare GAUGE with nine baselines and summarize their sources and configurations. The recovery-based baselines are MultiAE~\cite{ngiam2011multimodal}, MoPoE~\cite{sutter2021generalized}, M3Care~\cite{zhang2022m3care}, and
OnlineMAE~\cite{woo2023towards}; the recovery-free baselines are ModDrop~\cite{neverova2016moddrop}, MTL~\cite{ma2022missing}, MAP~\cite{lee2023multimodal}, and MUSE~\cite{wu2024muse}. We also include the dynamic-recovery baseline DyMo~\cite{du2026inference} with the cosine-distance reward, denoted as DyMo$_c$. GAUGE and our reproduced DyMo$_c$ share the same backbone and frozen imputation pipeline, making DyMo$_c$ the closest architecture-controlled baseline. Except for our reproduced DyMo$_c$, the results in Table~\ref{tab:missing-modality} are taken from Tables~1, S3, and S4 of \citet{du2026inference} under the same missingness protocols; $\dagger$ marks their incomplete-simulation variants (MultiAE, MoPoE, M3Care, and MAP).
Methods are reported only for the settings they support; unsupported entries are marked ``-''. Per-baseline configurations (e.g., the drop probability $p{=}0.5$ of ModDrop and MUSE, MAP prompt lengths, and the encoder initialization of OnlineMAE) are likewise unchanged from that work. Table~\ref{tab:missing-modality} focuses on methods explicitly designed or adapted for incomplete multimodal learning. Dynamic fusion methods such as QMF, DynMM, and PDF, which were evaluated by \citet{du2026inference} on observed-plus-recovered inputs, are not repeated in this table.

\subsection{Implementation Details}
\label{app:impl}
\noindent\textbf{Model Initialization.}
We adopt the dataset-specific backbones and frozen imputers from DyMo~\cite{du2026inference}. A MoPoE multimodal VAE is used to reconstruct missing modalities on
PolyMNIST, MST, and CelebA, while the TIP imputer~\cite{du2024tip} reconstructs missing tabular features on DVM,
CAD, and Infarction. The modality-specific encoders, multimodal Transformer, and classifier are initialized from our reproduced DyMo$_c$ checkpoints using the
architectures and hyperparameters reported in Table~S2 of \citet{du2026inference}. During fine-tuning, the encoders, Transformer, classifier, and gate
parameters \((\tau,\rho)\) are updated, whereas all imputer parameters remain frozen. We use the same dataset preprocessing as DyMo.

\noindent\textbf{Evidence-Unit Layouts.}
The evidence-unit layout is fixed by the inherited dataset-specific encoders. Each PolyMNIST view emits \(4\) units, giving \(N{=}20\) across its
five views. Each modality in MST and CelebA emits one global embedding, giving \(N{=}3\) and \(N{=}2\), respectively. 
DVM emits \(16\) image units and \(17\) tabular units (\(N{=}33\)), while CAD and Infarction each emit \(16\) image units and \(75\) tabular units (\(N{=}91\)). These layouts determine the branch used by the \(N\)-adaptive normalization rule (Appendix~\ref{app:norm-ablation}). Because each modality in MST and CelebA emits a single embedding, GAUGE reduces to modality-level gating on these two benchmarks.
\begin{table}[h]
\centering
\small
\setlength{\tabcolsep}{3.5pt}
\renewcommand{\arraystretch}{1.05}
\begin{tabular}{@{}lcccc@{}}
\toprule
Dataset & LR & Batch & Epochs & Metric \\
\midrule
PolyMNIST  & $10^{-3}$          & 256 & 100 & Acc. \\
MST        & $10^{-4}$          & 256 & 20  & Acc. \\
CelebA     & $10^{-3}$          & 256 & 20  & Acc. \\
DVM        & $10^{-4}$          & 256 & 300 & Acc. \\
CAD        & $10^{-3}$          & 128 & 300 & AUC \\
Infarction & $3{\times}10^{-4}$ & 128 & 300 & AUC \\
\bottomrule
\end{tabular}
\caption{Dataset-specific fine-tuning settings for GAUGE.}
\label{tab:training-settings}
\end{table}

\noindent\textbf{Optimization and Gate Settings.}
We fine-tune GAUGE using Adam with zero weight decay and the dataset-specific learning rates, batch sizes, and maximum numbers of epochs listed in Table~\ref{tab:training-settings}. A new missingness mask is sampled for each minibatch following the
protocol in Appendix~\ref{app:datasets}. We use the zero reference \(e'_{i,j}=\mathbf{0}\). The gate threshold \(\tau\) is initialized to \(0\) for PolyMNIST, DVM, CAD, and Infarction, which use \(z\)-score normalization, and to \(0.5\) for MST and CelebA, which use raw scores. For all benchmarks, \(\rho\) is initialized to \(0\), such that \(\exp(\rho)=1\), and we set \(\epsilon=10^{-6}\) and \(\epsilon_g=10^{-9}\). We select the checkpoint with the best validation accuracy on PolyMNIST, MST, CelebA, and DVM or the best validation AUC on CAD and Infarction.

\section{Supplementary Experiments}
\label{app:supp}
\subsection{Ablation on Evidence-Unit Scoring}
\label{app:scoring}
\begin{table*}[t]
\centering
\begin{tabular}{lccccccc}
\toprule
\multirow{2}{*}{Score Function}
& PolyMNIST & MST & CelebA & DVM & CAD & Infarction
& \multirow{2}{*}{Avg.} \\
& $\eta{=}0.8$
& $\{\mathrm{M},\mathrm{T}\}$
& $\{\mathrm{T}\}$
& $\gamma{=}1$
& $\gamma{=}1$
& $\gamma{=}1$
& \\
\midrule
Random score    & 97.14 & 85.01 & 92.56 & 94.22 & 66.74 & 70.37 & 84.34 \\
Feature norm    & 97.16 & 84.89 & 92.71 & 94.18 & 68.81 & 69.26 & 84.50 \\
Gradient norm   & 96.93 & 84.97 & 92.84 & 94.14 & 68.64 & 70.77 & 84.72 \\
Attention score & 96.87 & 85.03 & 92.76 & 94.29 & 67.92 & 69.33 & 84.37 \\
Taylor evidence score & \textbf{97.61} & \textbf{85.26} & \textbf{93.16} & \textbf{94.42} & \textbf{70.41} & \textbf{72.86} & \textbf{85.62} \\
\bottomrule
\end{tabular}
\caption{Ablation of evidence-unit scoring under the most severe
missingness setting for each benchmark. Accuracy (\%) is reported for
PolyMNIST, MST, CelebA, and DVM, while AUC (\%) is reported for CAD and
Infarction. All variants use the same score-to-gate pipeline and differ
only in the raw scoring function. \textbf{Bold} marks the best result.}
\label{tab:ablation-scoring}
\end{table*}

To examine whether GAUGE benefits specifically from its reference-based Taylor evidence score, we compare it with four alternative scoring functions: random score, feature norm, gradient norm, and attention score. For a controlled comparison, all variants use the same \(N\)-adaptive normalization, continuous gate, attention-bias injection, initialization, and fine-tuning settings, differing only in the raw evidence-unit score. For evidence unit \(e_{i,j}\), \emph{random score} samples an independent value from \(\mathcal{N}(0,1)\) at each scoring pass; \emph{feature norm} uses \(\|e_{i,j}\|_2\); and \emph{gradient norm} uses \(\|\nabla_{e_{i,j}}\mathcal{L}_{\mathrm{sal}}(\tilde{T}_i)\|_1\). \emph{Attention score} uses the final-layer
\([\mathtt{CLS}]\)-to-unit attention from the ungated forward pass, averaged across attention heads. The Taylor evidence score is the default score defined in Eq.~\eqref{eq:taylor_score}, using the zero reference
\(e'_{i,j}=\mathbf{0}\). As shown in Table~\ref{tab:ablation-scoring}, the Taylor evidence score
consistently outperforms all four alternatives across all six benchmarks, achieving the highest average performance of \(85.62\), which exceeds the strongest competing scorer, gradient norm, by
\(0.90\) points. These results support the effectiveness of GAUGE's reference-based Taylor scoring mechanism, showing that combining the prediction-aware gradient with the unit-to-reference displacement yields a stronger gating signal than the four alternative scores.

\subsection{Ablation on Score Normalization}
\label{app:norm-ablation}
GAUGE adopts an \(N\)-adaptive normalization rule because per-sample standardization behaves differently across encoder layouts. When only a few evidence units are available, the sample mean and variance are determined by too few scores, causing the standardized values to become highly constrained and discarding their absolute
magnitudes. For example, when \(N{=}2\), any two distinct scores are mapped to \(-1/\sqrt{2}\) and \(1/\sqrt{2}\), regardless of their original scale. In contrast, when many units are present, raw Taylor evidence scores can exhibit substantial sample-dependent offsets and scales, making the shared gate threshold and temperature difficult to calibrate consistently. To verify whether these two regimes require different treatments, we fix the raw score to the Taylor evidence score and compare the default \(N\)-adaptive rule in Eq.~\eqref{eq:zscore} with two fixed strategies: \emph{Always \(z\)-score} and \emph{Always raw}.
The comparison includes two small-\(N\) benchmarks, MST (\(N{=}3\)) and CelebA (\(N{=}2\)), and two large-\(N\) benchmarks, CAD and
Infarction (\(N{=}91\)). All other model components and training settings are kept unchanged. As shown in Table~\ref{tab:ablation-norm}, applying \(z\)-score normalization uniformly degrades performance in the small-\(N\) regime, reducing accuracy by \(0.23\) points on MST and \(0.58\)
points on CelebA relative to raw scores. This confirms that standardization based on only two or three units can over-constrain their relative values and remove useful magnitude information. Conversely, using raw scores uniformly reduces AUC by \(1.94\) points on CAD and \(1.67\) points on Infarction, showing that normalization is important for calibrating the scores of large evidence-unit sets. These results support the \(N\)-adaptive rule: raw scores are preferable for small \(N\), whereas per-sample standardization is beneficial for large \(N\).
\begin{table}[h]
\centering
\small
\renewcommand{\arraystretch}{1.15}
\resizebox{\columnwidth}{!}{%
\begin{tabular}{lcccc}
\toprule
\multirow{2}{*}{Normalization}
& MST
& CelebA
& CAD
& Infarction \\
& {\scriptsize $N{=}3$}
& {\scriptsize $N{=}2$}
& {\scriptsize $N{=}91$}
& {\scriptsize $N{=}91$} \\
\midrule
Always $z$-score
& 85.03
& 92.58
& \textbf{70.41}
& \textbf{72.86} \\
Always raw
& \textbf{85.26}
& \textbf{93.16}
& 68.47
& 71.19 \\
\textbf{$N$-adaptive (default)}
& \textbf{85.26}
& \textbf{93.16}
& \textbf{70.41}
& \textbf{72.86} \\
\bottomrule
\end{tabular}%
}
\caption{Ablation of score normalization under each benchmark's most severe
missingness setting. Accuracy (\%) is reported for MST and CelebA, while
AUC (\%) is reported for CAD and Infarction. \textbf{Bold} marks the best
result.}
\label{tab:ablation-norm}
\end{table}

\subsection{Ablation on Evidence-Unit Granularity}
\label{app:evidence-units}
We examine whether finer-grained image evidence units improve GAUGE's unit-level control. A finer spatial grid exposes more localized regions as separate evidence units, each with its own Taylor evidence score and gate. We test this by varying the input resolution on CAD and Infarction at \(\gamma=1\). Input resolutions of \(96\times96\), \(128\times128\), \(160\times160\), \(192\times192\), and \(256\times256\) pixels yield \(3\times3\), \(4\times4\), \(5\times5\), \(6\times6\), and \(8\times8\) evidence-unit grids, respectively. The native \(128\times128\) input corresponds to a \(4\times4\) grid. We keep the model architecture, optimization setup, and gating configuration fixed across all resolution settings. Only the image positional embeddings are reinitialized, while all other pre-trained weights are retained and fine-tuned. Fig.~\ref{fig:evidence_units} shows that finer grids generally improve test AUC on both benchmarks. CAD AUC increases from \(69.10\) at \(3\times3\) to \(70.87\) at \(8\times8\), while Infarction reaches its highest AUC of \(73.90\) at \(6\times6\), with no further gain at \(8\times8\).
Because bilinear upsampling adds no new source-pixel information, the gains beyond the native \(4\times4\) grid suggest that finer evidence-unit granularity can improve localized scoring and gating. On Infarction, performance plateaus beyond \(6\times6\), indicating limited benefit from further grid refinement.

\begin{figure}[!h]
\centering
\includegraphics[width=\columnwidth]{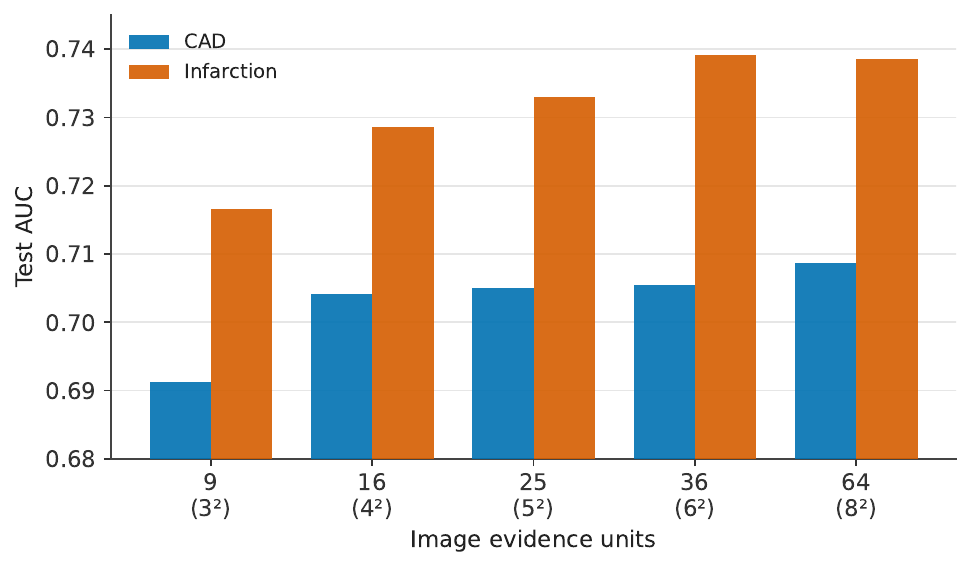}
\caption{Effect of image resolution and the resulting evidence-unit
granularity on CAD and Infarction at \(\gamma=1\).}
\label{fig:evidence_units}
\end{figure}

\section{Scoring Cost Comparison}
\label{app:cost}
Table~\ref{tab:cost} reports the runtime of the scoring stage, measured from the encoded sequence \(\tilde{T}_i\) to the \(N\) evidence-unit scores. We exclude imputation, modality encoding, and the final gated forward
pass, as these operations are identical for both methods. Exact counterfactual scoring performs one ungated forward pass followed by \(N\) sequential intervention forward passes, with batch size \(1\). In contrast, GAUGE computes all \(N\) Taylor evidence scores through a single ungated forward--backward pass. For each configuration, runtime is averaged over \(500\) test samples after \(20\) warm-up iterations. We call \texttt{torch.cuda.synchronize()} before and after each measurement to account for asynchronous GPU execution. All timing measurements are conducted on a single NVIDIA Quadro RTX~8000.

\end{document}